%% file: smoothness.tex
\documentclass{article}

\usepackage[numbers]{natbib}
\usepackage{enumerate}

\usepackage{fullpage}
\usepackage{url}

\include{header}

\include{notation}

\newcommand{\mc}[1]{\mathcal{#1}}
\newcommand{\mbb}[1]{\mathbb{#1}}
\newcommand{\mbf}[1]{\mathbf{#1}}
\newcommand{\hL}{\hat{L}}
\renewcommand{\L}[1]{L\left(#1\right)}
\newcommand{\cL}{\mc{L}}
\newcommand{\hw}{\hat{\w}}
\newcommand{\eh}{\hat{h}}
\renewcommand{\Rad}{\mc{R}}
\renewcommand{\H}{\mc{H}}
\renewcommand{\Z}{\mc{Z}}

\newcommand{\hRad}{\mc{R}}
\newcommand{\fat}{\mathrm{fat}}
\newcommand{\err}{\mathrm{err}}
\newcommand{\herr}{\widehat{\mathrm{err}}}
\newcommand{\Es}[2]{\mbb{E}_{#1}\left[ #2 \right]}

\begin{document}
\setlength{\parskip}{2mm}
\setlength{\parindent}{0pt} 

\title{Optimistic Rates for Learning with a Smooth Loss}
\author{
Nathan Srebro \hspace{0.5in} Karthik Sridharan \\ 
Toyota Technological Institute at Chicago\\ 6045 S Kenwood Ave., Chicago, IL 60637
\vspace{0.3in} \\
Ambuj Tewari \\
Department of Computer Science \\
University of Texas at Austin }

\date{}
\maketitle

\begin{abstract}
  We establish an excess risk bound of $\widetilde{O}\left(H \Rad_n^2 +
    \sqrt{H L^*}\Rad_n\right)$ for empirical risk minimization with an $H$-smooth loss
  function and a hypothesis class with Rademacher complexity $\Rad_n$,
  where $L^*$ is the best risk achievable by the hypothesis class.
  For typical hypothesis classes where $\Rad_n = \sqrt{R/n}$, this
  translates to a learning rate of $\widetilde{O}\left(RH/n\right)$ in the
  separable ($L^*=0$) case and $\widetilde{O}\left(RH/n + \sqrt{L^*
      RH/n}\right)$ more generally.
  We also provide similar
  guarantees for online and stochastic convex optimization with a
  smooth non-negative objective.
\end{abstract}

\section{Introduction}

Consider empirical risk minimization for a hypothesis class $\H = \{ h:\X
\rightarrow \R\}$ with respect to some non-negative loss function $\phi(t,y)$.
That is, we would like to learn a predictor $h$ with small risk
$$\L{h}=\E{\phi(h(X),Y)}$$ by minimizing the empirical risk
$$\hL(h)=\frac{1}{n}\sum_{i=1}^n\phi(h(x_i),y_i)$$
given an i.i.d.~sample $(x_1,y_1),\ldots,(x_n,y_n)$.

Statistical guarantees on the excess risk are well understood for {\em
  parametric} (i.e.~finite dimensional) hypothesis classes.  More
formally, these are hypothesis classes with finite VC-subgraph
dimension \citep{Pollard84} (also known as the pseudo-dimension).  For such classes, learning
guarantees can be obtained for any bounded loss function
(i.e. any $\phi$ such that $\abs{\phi}\leq b < \infty$) and the relevant measure of
complexity is the VC-subgraph dimension.

Alternatively, even for some non-parametric hypothesis classes
(i.e.~those with infinite VC-subgraph dimension), e.g.~the class of
low-norm linear predictors
$$\H_B =\left\{ h_{w}:\x\mapsto\ip{\w}{\x} \;\middle|\; \norm{\w}_2 \leq B
\right\},$$ guarantees can be obtained in terms of {\em
  scale-sensitive} measures of complexity such as fat-shattering
dimensions \citep{AlonBeCeHa93}, covering numbers \citep{Pollard84} or Rademacher complexity
\citep{BartlettMe02}.  The classical statistical learning theory approach for
obtaining learning guarantees for such scale-sensitive classes is to
rely on the Lipschitz constant $D$ of $\phi(t,y)$ with respect to its first argument $t$.
If the loss is differentiable then this amounts to an upper bound
on the magnitude of the first derivative with respect to $t$.
The excess risk can then
be bounded as (expectation here is over the sample):
\begin{align}
  \E{\L{\hat{h}}} &\leq L^* + 2 D \Rad_n(\H) \notag \\
 &= L^* + 2 \sqrt{D^2 \frac{R}{n}}   \label{eq:lip}
\end{align}
where $\hat{h}=\arg\min_h \hL(h)$ is the empirical risk minimizer
(ERM), $L^*=\inf_h \L{h}$ is the minimal possible risk in $\H$, and
$\Rad_n(\H)$ is the Rademacher complexity of the class $\H$.  The
Rademacher complexity typically scales as $\Rad_n(\H)=\sqrt{R/n}$,
yielding the expression on the second line.  For instance, in the case
of $\ell_2$-bounded linear predictors, $R = B^2 \norm{X}_2^2$ where
$\norm{X}_2 = \sup_{\x\in\X} \norm{\x}_2$.  The Rademacher complexity
can be bounded by other scale-sensitive complexity measures, such as
the fat-shattering dimensions and covering numbers, yielding similar
guarantees in terms of these measures.

In this paper, we address two deficiencies of the guarantee
\eqref{eq:lip}.  

First, the bound applies only to loss functions with
bounded derivative, like the hinge and logistic losses (popular for
classification), or the absolute-value loss (for regression).
It is not directly applicable to the squared loss
$\phi(t,y)={\scriptstyle\tfrac{1}{2}}(t-y)^2$, for which the {\em second}
derivative is bounded, but not the first.  We could try to simply
bound the derivative of the squared loss in terms of a bound on the
magnitude of $h(x)$, but for norm-bounded linear predictors
$\H_B$, for instance, this results in a very disappointing excess risk bound of the
form $O(\sqrt{ B^4 \norm{X}_2^4 / n })$.  One aim of this paper
is to provide clean bounds on the excess risk for smooth loss
functions, such as the squared loss, with a bounded \emph{second}, rather then
first, derivative.

The second deficiency of \eqref{eq:lip} is the dependence on
the sample size $n$.  The $1/\sqrt{n}$ dependence might be unavoidable in
general.  But at least for finite dimensional (parametric) classes, we
know it can be improved to a $1/n$ rate when the distribution is
separable\footnote{Several binary classification losses evaluate to zero when the ``margin" $yh(x)$ is sufficiently positive.
For such losses, if the distribution is separable by some $h^* \in \H$ with enough margin, we have $L^* = 0$. This explains why
we use the term ``separable" to denote $L^* = 0$ for general losses.}, i.e.~when there exists $h \in \H$ with $\L{h}=0$ and so
$L^*=0$.  In particular, if $\H$ is a class of bounded functions with
VC-subgraph-dimension $d$
(e.g.~$d$-dimensional linear predictors), then \citep{Panchenko02}:
\begin{equation}
  \label{eq:panchenko}
  \E{\L{\hat{h}}} \leq L^* + O\left( \frac{ d D \log n}{n} + \sqrt{ \frac{d D L^* \log n}{n}} \right) \ .
\end{equation}
Notice that the $1/\sqrt{n}$ term disappears in the separable case, and we get a graceful degredation between the $1/\sqrt{n}$ non-separable rate and the $1/n$ separable rate.  The sample complexity (number of samples needed to guarantee that excess risk is smaller than $\epsilon$) associated with the learning rate above is given by:
$$
n = O\left(\frac{d D}{\epsilon} \left(\frac{L^* + \epsilon}{\epsilon}\right)  \log\left(\frac{d D}{\epsilon}\right)\right) \ .
$$
In the separable case, when $L^*=0$, as well as in the non-separable
case as long as we are concerned with excess error $\epsilon$ which is
not much smaller then $L^*$ (roughly speaking, an estimation
error not much smaller than the optimal risk), the term
$(L^*+\epsilon)/\epsilon$ can be thought of as constant, and the
sample complexity scales roughly as $1/\epsilon$.  Only when we seek
excess error $\epsilon$ much smaller then $L^*$, might we get a
$1/\epsilon^2$ scaling.  We refer to such a rate as an ``optimistic
rate''.\footnote{We have borrowed the term ``optimistic'' for the
  rates we provide from Dmitry Panchenko's lecture notes.}  

As we will show, the two deficiencies are actually related.  For
non-parametric classes, and non-smooth Lipschitz loss, such as the
hinge-loss, the excess risk might scale as $1/\sqrt{n}$ and not $1/n$,
{\em even in the separable case}.  However, for $H$-smooth non-negative loss functions, where the second derivative of $\phi(t,y)$ with respect to $t$ is
bounded by $H$, a $1/n$ separable rate {\em is} possible.  In Section
\ref{sec:rad} we obtain the following bound (with high probability) on the excess risk (up to
logarithmic factors):
\begin{align}
  \label{eq:intromain}
  \L{\hat{h}} &\leq L^* + \widetilde{O}\left( H\Rad^2_n(\H) + \sqrt{H L^*}
    \Rad_n(\H) \right) \notag \\
  &= L^* + \widetilde{O}\left( \frac{H R}{n} +
    \sqrt{\frac{H R L^*}{n}} \right) \;\;\leq\;\; 2L^* +\tilde{O}\left(\frac{H R}{n}\right).
\end{align}
where again the second line corresponds to the typical scaling
$\Rad_n(\H)=\sqrt{R/n}$.  In this case, we obtain the following bound
on the sample complexity required for excess error $\L{\hat{h}} \leq
L^* + \epsilon$:
$$
n \le O\left( \frac{R}{\epsilon}\left(\frac{L^* + \epsilon}{\epsilon} \right) \log^{3}(R/\epsilon) \right) \ .
$$
In particular, for $\ell_2$-norm-bounded linear predictors $\H_B$ with $\norm{X}_2^2 \leq 1$, the excess risk is bounded by $\widetilde{O}(HB^2/n + \sqrt{HB^2L^*/n})$.  Another interesting distinction between
parametric and non-parametric classes is that, even for the
squared-loss, the bound \eqref{eq:intromain} is tight and the
non-separable rate of $1/\sqrt{n}$ is unavoidable.  This is in
contrast to the parametric (finite dimensional) case, where a rate of
$1/n$ is always possible for the squared loss, regardless of the
value of $L^*$ \citep{LeeBaWi98}.  The differences between
parametric and scale-sensitive classes, and between non-smooth, smooth
and strongly convex (e.g.~squared) loss functions are discussed in
Section \ref{sec:tight} and summarized in Table \ref{tab:tight}.

The guarantees discussed thus far are general learning guarantees for
the stochastic setting that rely only on the Rademacher complexity of
the hypothesis class, and are phrased in terms of minimizing some
scalar loss function.  In Section \ref{sec:convex}, we consider also
the online setting, in addition to the stochastic setting, and present
similar guarantees for online and stochastic convex optimization
\citep{Zinkevich03,ShalevShSrSr09}. The guarantees of Section
\ref{sec:convex} match equation \eqref{eq:intromain} for the special
case of a convex loss function and norm-bounded linear predictors, but
Section \ref{sec:convex} captures a more general setting of optimizing
an arbitrary non-negative smooth convex objective (there is no separate discussion of a ``predictor'' and a
scalar loss function in Section \ref{sec:convex}). Results in Section
\ref{sec:convex} are expressed in terms of properties of the norm,
rather then a measure of statistical complexity like the Radamacher complexity
as in \eqref{eq:intromain} and Section \ref{sec:rad}.  However, the
online and stochastic convex optimization setting of Section
\ref{sec:convex} is also more restrictive, as we require the objective
to be convex (while for the bound
\eqref{eq:intromain} we make no assumption about the convexity of the
hypothesis class $\H$ nor the loss function $\phi$).

Specifically, for a non-negative $H$-smooth {\em convex} objective
(see exact definition in Section \ref{sec:convex}), over a domain
bounded by $B$, we prove that the average online regret (and so also
the excess risk of stochastic optimization) is bounded by $O(HB^2/n +
\sqrt{HB^2L^*/n})$.  Comparing with the bound of $O(\sqrt{D^2 B^2/n})$
when the loss is $D$-Lipschitz rather then $H$-smooth
\citep{Zinkevich03,NemirovskiYu78}, we see the same relationship
discussed above for ERM.  Unlike the bound \eqref{eq:intromain} for
the ERM, the convex optimization bound avoids polylogarithmic factors.
The results in Section \ref{sec:convex} also generalize to smoothness
and boundedness with respect to non-Euclidean norms.

Studying the online and stochastic convex optimization setting
(Section \ref{sec:convex}), in addition to ERM (Section
\ref{sec:rad}), has several advantages.  First, it allows us to obtain
a learning guarantee for an efficient single-pass learning method,
namely stochastic gradient descent (or mirror descent), as well as for non-stochastic regret.  Second, the bound we obtain in the convex
optimization setting (Section \ref{sec:convex}) is actually better
then the bound for the ERM (Section \ref{sec:rad}) as it avoids all
polylogarithmic and large constant factors.  Third, the bound is
applicable to other non-negative online or stochastic optimization
problems beyond classification, including problems for which ERM is
not applicable (see, e.g.,~\citep{ShalevShSrSr09}).

In order to establish our main result we go back and forth between covering numbers and Radamacher complexity, and for this purpose we include in Appendix A results
establishing tight relationships between the various complexity measures. These results might also be of independent interest to readers.

\section{Empirical Risk Minimization with a Smooth Loss}\label{sec:rad}

Recall that the worst-case Rademacher complexity \citep{BartlettMe02} of $\H$ for any $n \in \mathbb{N}$ is given by:
\begin{equation}
  \label{eq:raddef}
\hRad_n(\H) = \sup_{x_1,\ldots , x_n \in \X} \Es{\sigma \sim \text{Unif}(\{\pm 1\}^n)}{\sup_{h \in \H} \frac{1}{n} \left|\sum_{i=1}^n  h(x_i) \sigma_i \right|}.
\end{equation}
Throughout, we shall consider this ``worst case" Rademacher complexity.

Our starting point is the learning bound \eqref{eq:lip} that applies
to $D$-Lipschitz loss functions, i.e. such that $\abs{\phi'(t,y)}\leq D$ (we
always take derivatives with respect to the first argument).  What type of
bound can we obtain if we instead bound the second derivative
$\phi''(t,y)$?  We will actually avoid talking about the second
derivative explicitly, and instead say that a function is $H$-smooth
iff its derivative is $H$-Lipschitz.  For twice differentiable $\phi$,
this just means that $\abs{\phi''}\leq H$.  The central observation, which
allows us to obtain  guarantees for smooth loss functions, is that for
a smooth loss, the derivative can be bounded in terms of the function value:
\begin{lemma}\label{lemma:smooth}
For an $H$-smooth non-negative function $f : \mathbb{R} \mapsto \mathbb{R}$, we have:
$$\abs{f'(t)}  \le \sqrt{4 H f(t)}\ .$$
\end{lemma}
\begin{proof}
For any $t < r$, there is an $s \in [t,r]$ for
which $f(r) = f(t) + f'(s)(r- t)$.  Now:
\begin{align*}
0 & \le  f(r) = f(t) + f'(t)(r - t) + (f'(s) - f'(t))(r- t) \\
& \le f(t) + f'(t)(r - t) + H \abs{s-t}\abs{r - t} \le f(t) + f'(t)(r - t) + H (r - t)^2
\end{align*} 
Setting $r = t - \frac{f'(t)}{2 H}$ yields the desired bound.
\end{proof}

The above lemma allows us to argue that close to the optimum value, where
the {\em value} of the loss is small, then so is its derivative.
Looking at the dependence of \eqref{eq:lip} on the derivative bound
$D$, we are guided by the following heuristic argument : since we
should be concerned only with the behavior around the ERM, perhaps it
is enough to bound $\phi'(\hat{h}(x),y)$ at the ERM $\hat{h}$.  Applying
Lemma \ref{lemma:smooth} to $L(\hat{h})$, we can bound
$\abs{\E{\phi'(\hat{h}(x),y)}} \leq \sqrt{4 H L(\hat{h})}$.  What we
would actually want is to bound each $\abs{\phi'(\hat{\w},x)}$
separately, or at least have the absolute value {\em inside} the
expectation---this is where the non-negativity of the loss plays an
important role.  Ignoring this important issue for the moment and
plugging this instead of $D$ into \eqref{eq:lip} yields $L(\hat{h})
\leq L^* + 4\sqrt{H L(\hat{h})} \Rad_n(\H)$.  Solving for $L(\hat{h})$
yields the desired bound \eqref{eq:intromain}.

This rough intuition is captured by the following theorem.

\begin{theorem}\label{thm:mainRad}
  Let $\phi$ be an $H$-smooth non-negative loss s.t.  $\forall {\hat{y},\hat{y}',y},\ \abs{\phi(\hat{y},y) - \phi(\hat{y}',y)}\leq
  b$. Then, for any $\delta > 0$ we have, with probability at least $1 -
  \delta$ over a random sample of size $n$, for any $h \in \H$,
{\small
\begin{align*}
\L{h} \le \hL(h) +K\left( \sqrt{\hL(h)} \left(\sqrt{H} \log^{1.5}\! n\  \hRad_n(\H)  + \sqrt{\frac{b \log(1/\delta)}{n}}\right)+ H \log^{3}\!n\  \hRad_n^2(\H)  + \frac{b \log(1/\delta)}{n}\right)
\end{align*}
and so:
\vspace{-0.1in}
\begin{align*}
\L{\eh} \le L^* + K \left( \sqrt{L^*} \left(\sqrt{H} \log^{1.5}\!n\ \hRad_n(\H)  + \sqrt{\frac{b \log(1/\delta)}{n}}\right)+  H \log^{3}\!n\  \hRad_n^2(\H)  + \frac{b \log(1/\delta)}{n}\right)
\end{align*}
}
where $K < 10^5$ is a numeric constant derived from \citep{Mendelson02} and \citep{Bousquet02}. 
\end{theorem}
Note that only the ``confidence'' terms depended on $b=\sup
\abs{\phi}$, and this is typically not the dominant term. We
believe it is possible to also obtain a bound that holds in
expectation over the sample (rather than with high probability) and
that avoids a direct dependence on $\sup\abs{\phi}$.

The following simple corollary of the above theorem bounds the sample complexity of learning with smooth loss functions.
\begin{corollary}
Assume that for any $n \ge 1$ the Rademacher complexity of function class $\mathcal{H}$ can be bounded as $\hRad_n(\H) \le \sqrt{\frac{R}{n}}$. Then given any $H$-smooth non-negative loss $\phi$ bounded by $b$ and any $\delta, \epsilon > 0$, the number of samples $n$, required to guarantee that with probability at least $1 - \delta$, $L(\hat{h}) - L^* \le \epsilon$ is bounded as
\begin{align*}
n \le O\left( \left(\frac{R \log^{3}(R/\epsilon)  + b \log(1/\delta)}{\epsilon}\right) \left(\frac{L^* + \epsilon}{\epsilon} \right) \right) \ .
\end{align*}
\end{corollary}
\begin{remark} With slight modifications in the proof, one can replace the $\log^3(R/\epsilon)$ term above with $\log^3(B/\epsilon)$ where $B  := \sup_{\x \in \X, h \in \H} h(x)$ is the bound on functions in the hypothesis class $\H$.\end{remark}

To prove Theorem \ref{thm:mainRad}, we use the notion of local
Rademacher complexity \citep{BartlettBoMe05}, which allows us to focus
on the behavior of $\H$ in the vicinity of the ERM.  To this end, consider the following
empirically restricted loss class
$$
\cL_\phi(r) := \left\{(x,y) \mapsto \phi(h(x),y) : h \in \H , \hL(h) \le r\right\} \ .
$$
Lemma \ref{lem:key} presented below, is the key to the proof of the main theorem and solidifies the heuristic intuition discussed above. It shows that the Rademacher complexity of class $\cL_\phi(r)$ scales as $\sqrt{Hr}$.  The lemma can be seen as a higher-order version of the Lipschitz composition lemma \citep{BartlettMe02}, which states that the Rademacher complexity of the {\em unrestricted} loss class is bounded by $D \Rad_n(\H)$.  Here, we use the second, rather then first, derivative, and obtain a bound that depends on the empirical restriction:

\begin{lemma}\label{lem:key}
For a non-negative $H$-smooth loss and any function class $\H$, we have: 
\begin{align*}
\hRad_n(\cL_\phi(r)) & \le  21 \sqrt{6 H r}\,  \log^{\frac{3}{2}}\left(64\, n\right)\, \Rad_n(\mathcal{H}) \ .
\end{align*}
\end{lemma} 
\textit{Proof outline for Lemma \ref{lem:key}. }
We delay the detailed proof of the lemma to the appendix and provide an outline of the proof here. In order to prove the lemma, we actually move from Rademacher complexity to covering numbers, use smoothness and Lemma \ref{lemma:smooth} to obtain an $r$-dependent cover of the empirically restricted class, and then return to the Rademacher complexity.  More specifically the proof is outlined as follows :
\vspace{-0.18in}
\begin{enumerate}
  \setlength{\topsep}{0pt}
  \setlength{\itemsep}{1pt}
  \setlength{\parskip}{0pt}
  \setlength{\parsep}{0pt}
\item We use a modified version of Dudley's integral to bound the
  Rademacher complexity of the empirically restricted loss class in terms
  of the $L_2$-covering numbers of the class.
\item We use smoothness to get an $r$-dependent bound on the
  $L_2$-covering numbers of the empirically restricted loss class in
  terms of $L_{\infty}$-covering numbers of the unrestricted
  hypothesis class.
\item We bound the $L_{\infty}$-covering numbers of the unrestricted
  class in terms of its fat-shattering dimension, which in turn can be
  bounded in terms of its Rademacher complexity.
\end{enumerate}

\begin{proof}[\bf Proof of Theorem \ref{thm:mainRad}]
By Theorem 6.1 of \citep{Bousquet02} (specifically the displayed equation prior to the last one in the proof of the theorem) we have that if $\psi_n$ is any sub-root function that satisfies for all $r > 0$,  $\hRad_n(\cL_\phi(r)) \le \psi_n(r)$ then, for any $\delta >0$, with probability at least $1 - \delta$, for any $h \in \H$, 
\begin{equation} \label{eq:olivier}
\L{h}  \le \hL(h) + 45 r^*_n + \sqrt{\L{h}} \left( \sqrt{8 r^*_n} + \sqrt{\frac{4 b(\log\left(\tfrac{1}{\delta}\right) + 6 \log \log n )}{n}} \right) + \frac{20 b(\log\left(\tfrac{1}{\delta}\right)  + 6 \log \log n )}{n}
\end{equation}
where $r^*_n$ is the largest solution to equation $\psi_n(r) = r$.
Now by Lemma \ref{lem:key} we have that $\psi_n(r)=21 \sqrt{6 H r} \log^{1.5} n \hat{\mathcal{R}}_n(\H)$ satisfies the property that for all $r > 0$,  $\hRad_n(\cL_\phi(r)) \le \psi_n(r)$ and so using this we see that 
$$
r^*_n = 2646 H \log^3(64\, n) \hRad^2_n(\H)
$$
and for this $r^*_n$, the upper bound~\eqref{eq:olivier} holds. Now using the simple fact that for any non-negative $A,B,C$,  
$$
A \le B + C\sqrt{A} \Rightarrow A \le B + C^2 + \sqrt{B}C
$$
we conclude,
\begin{equation}\label{eq:loct}
\L{h} \le \hL(h) + 106\ r^*_n  + \frac{48 b}{n}  \left(\log\tfrac{1}{\delta} + \log \log\ n \right) + \sqrt{\hL(h) \left(8 r^{*}_n + \frac{4 b}{n}  \left(\log\tfrac{1}{\delta} + \log \log\ n \right) \right)}  \ .
\end{equation}
Now we claim that $\frac{4 b \log \log\ n}{n} \le 0.049 r^*_n$. To see this first note that by definition of $b$,
\begin{align*}
b & = \max_{y,\hat{y},\hat{y}'} \left(\phi(\hat{y},y) - \phi(\hat{y}',y)\right) \le \max_{y,\hat{y},\hat{y}'} \left|\phi'(\hat{y},y)\right| \left|\hat{y} - \hat{y}'\right|
\end{align*}
Now notice that in the proof of Lemma \ref{lemma:smooth} we in fact first showed that $|f'(t)| \le \sqrt{4 H (f(t) - f(r))}$ for any $r > t$ and only then using the fact that $f$ is non-negative we concluded that $|f'(t)| \le \sqrt{4 H f(t)}$. Hence we can conclude that $\left|\phi'(\hat{y},y)\right| \le \sqrt{4 H b}$. Hence using this in the above inequality we can conclude that 

$$
b  \le  4 H \max_{\hat{y},\hat{y}'}(\hat{y} - \hat{y}')^2 \le 16 H \max_{\hat{y}} |\hat{y}|^2 = 16 H \max_{x, h \in \H} |h(x)|^2
$$ 
Now on the other hand by definition of Rademacher complexity and by Khintchine's inequality we have that $\Rad_n(\H) \ge \sup_{x,y, h \in \H} |\phi(h(x),y)|/\sqrt{2n}$. Thus we have shown that 
$$
\frac{4 b \log \log\ n}{n}  \le \frac{64 H \sup_{x,y, h \in \H} |\phi(h(x),y)|^2 \log \log n }{n} \le 128 H \log \log n\  \Rad^2_n(\H) \le 0.049 r^*_n
$$
Plugging this back in Equation \ref{eq:loct} we see that 
\begin{align*}
\L{h} \le \hL(h) + 109\ r^*_n  + \frac{48 b \log\tfrac{1}{\delta}}{n}   + \sqrt{\hL(h) \left(9 r^{*}_n + \frac{4 b \log\tfrac{1}{\delta} }{n}  \right)}  \ .
\end{align*}
Plugging in the value of $r^*_n = 2646 H \log^3(64\, n) \hRad^2_n(\H)$ we get the first inequality.  To get the second inequality, we simply use the first inequality with the ERM $\hat{h}$ and further note that $\hat{L}(\hat{h}) \le \hat{L}(h^*)$ (where $h^*$ is  $\argmin{h \in \H}L(h)$). This gives us a bound of
\begin{align}\label{eq:almost}
\L{\hat{h}} \le \hL(h^*) + 109\ r^*_n  + \frac{48 b \log\tfrac{1}{\delta} }{n}  + \sqrt{\hL(h^*) \left(9 r^{*}_n + \frac{4 b \log\tfrac{1}{\delta}}{n}   \right)}  \ .
\end{align}
Now to conclude the proof notice that by Bernstein's inequality, with probability at least $1 - \delta$  :
\begin{align}
 \hL(h^*) - L^* & \le \sqrt{\frac{4 \E{(\phi(h^*(x),y) - \L{h^*})^2} \log\tfrac{1}{\delta}}{n}} + \frac{4 b \log\tfrac{1}{\delta}}{n} \notag \\
& \le \sqrt{\frac{8 b \L{h^*} \log\tfrac{1}{\delta}}{n}} + \frac{4 b \log\tfrac{1}{\delta}}{n} \label{eq:gt}
\end{align}
Hence using the above in Equation \ref{eq:almost}
we get that 
\begin{align*}
\L{\hat{h}} \le L^* + 109\ r^*_n  + \frac{52 b \log\tfrac{1}{\delta} }{n}  + \sqrt{\hL(h^*) \left(9 r^{*}_n + \frac{4 b \log\tfrac{1}{\delta}}{n}   \right)} +  \sqrt{\frac{8 b L^* \log\tfrac{1}{\delta}}{n}}  \ .
\end{align*}
Again Equation \ref{eq:gt} implies that with probability $1 - \delta$, $\hat{L}(h^*) \le \frac{3}{2} L^* + \frac{8 b  \log\tfrac{1}{\delta}}{n} $ and so using this in the above we conclude that
\begin{align*}
\L{\hat{h}} \le L^* + 109\ r^*_n  + \frac{52 b \log\tfrac{1}{\delta} }{n}  + \sqrt{\left( \frac{3}{2} L^* + \frac{8 b  \log\tfrac{1}{\delta}}{n} \right) \left(9 r^{*}_n + \frac{4 b \log\tfrac{1}{\delta}}{n}   \right)} +  \sqrt{\frac{8 b L^* \log\tfrac{1}{\delta}}{n}}  \ .
\end{align*}
Plugging in $r^*_n$ and over bounding with appropriate numeric constant $K$ concludes the proof.
\end{proof}

\subsection{Related Results}

Rates faster than $1/\sqrt{n}$ have been previously explored under
various conditions, including when $L^*$ is small. 

\paragraph{The Finite Dimensional Case}

\cite{LeeBaWi98} showed faster rates for squared loss,
exploiting the strong convexity of this loss function, even when
$L^*>0$, but only with finite VC-subgraph-dimension.
\cite{Panchenko02} provides optimistic rate results for general Lipschitz
bounded loss functions, still in the finite VC-subgraph-dimension
case.  \cite{Bousquet02} provided similar guarantees for
linear predictors in Hilbert spaces when the spectrum of the kernel
matrix (covariance of $X$) is exponentially decaying, making the
situation almost finite dimensional.  All these methods rely on
finiteness of effective dimension to provide fast rates.  In this
case, smoothness is not necessary.  Our method, on the other hand,
establishes optimistic rates (and a fast rate when $L^*=0$), for function classes that do
{\em not} have finite VC-subgraph-dimension.  In Section \ref{sec:tight} We show how in the
non-parametric case, smoothness is necessary for optimistic rates and how it plays an important
role (see also Table \ref{tab:tight}).

\paragraph{Aggregation}

\cite{Tsybakov04} studied learning rates for aggregation,
where a predictor is chosen from the convex hull of a finite set of
base predictors.  This is equivalent to an $\ell_1$ constraint where
each base predictor is viewed as a ``feature''.  As with
$\ell_1$-based analysis, since the bounds depend only logarithmically
on the number of base predictors (i.e. dimensionality), and rely on
the scale of change of the loss function, they are of a ``scale
sensitive'' nature.  For such an aggregate classifier, Tsybakov
obtained a rate of $1/n$ when zero (or small) risk is achieved by one
of the base classifiers.  In Tsybakov's result, it is not enough
to assume that zero risk is achieved by an aggregate (i.e.~bounded $\ell_1$)
classifier in order to obtain the faster rate.  Tsybakov's core result
is thus in a sense more similar to the finite dimensional results,
since it allows for a rate of $1/n$ when zero error is achieved by a
finite cardinality (and hence finite dimension) class.  

Tsybakov then used the approximation error of a small class of
base predictors with respect to a large hypothesis class (i.e.~a covering) to
obtain learning rates for the large hypothesis class by considering
aggregation within the small class.  However these results only imply
fast learning rates for hypothesis classes with very low complexity.
Specifically, to get learning rates better than $1/\sqrt{n}$ using
these results, the covering number of the hypothesis class at scale
$\epsilon$ needs to behave as $1/\epsilon^p$ for some $p < 2$.  But
typical classes, including the class of linear predictors with bounded
norm, have covering numbers that scale as $1/\epsilon^2$ and so these
methods do not imply fast rates for such function classes.  In fact,
to get rates of $1/n$ with these techniques, even when $L^*=0$,
requires covering numbers that do not increase with $\epsilon$ at all,
and so actually requires finite VC-subgraph-dimension.

\cite{ChesneauLe06} extend Tsybakov's work also to
general losses, deriving similar results for Lipschitz loss function.
The same caveats hold: even when $L^*=0$, rates faster when
$1/\sqrt{n}$ require covering numbers that grow slower than
$1/\epsilon^2$, and rates of $1/n$ essentially require finite
VC-subgraph-dimension.  Our work, on the other hand, is applicable
whenever the Rademacher complexity (equivalently covering numbers) can
be controlled.  Although it uses some similar techniques, it is also
rather different from the work of Tsybakov and Chesneau et al., in that
it points out the importance of smoothness for obtaining fast rates in
the non-parametric case: Chesneau et al. relied only on the Lipschitz
constant, which we show, in Section \ref{sec:tight}, is not enough for
obtaining fast rates in the non-parametric case, even when $L^*=0$.

\paragraph{Local Rademacher Complexities}

\cite{BartlettBoMe05} developed a general machinery for
proving possible fast rates based on local Rademacher complexities.
However, it is important to note that the localized complexity term
typically dominates the rate and still needs to be controlled.  For
example, \cite{Steinwart07} used local Rademacher complexity
to provide fast rate on the 0/1 loss of Support Vector Machines (SVMs)
($\ell_2$-regularized hinge-loss minimization) based on the so called
``geometric margin condition'' and Tsybakov's margin condition.
Steinwart's analysis is specific to SVMs.  We also use local
Rademacher complexities in order to obtain fast rates, but do so for
general hypothesis classes, based only on the standard Rademacher
complexity $\hRad_n(\H)$ of the hypothesis classes, as well as the
smoothness of the loss function and the magnitude of $L^*$, but
without any further assumptions on the hypothesis classes itself.

\paragraph{Non-Lipschitz Loss}

We are not aware of prior work providing an
explicit and easy-to-use result for controlling a generic
non-Lipschitz loss (such as the squared loss) solely in terms of the
Rademacher complexity.

\section{A Sharp Understanding of Slow, Optimistic and Fast Rates}\label{sec:tight}

In this section we look at learning rates for the ERM for
parametric and for scale-sensitive hypothesis classes (i.e.~in terms
of the dimensionality and in terms of scale sensitive complexity
measures), discussed in the Introduction and analyzed in Section
\ref{sec:rad}.  We compare the guarantees on the learning rates in
different situations, identify differences between the parametric and
scale-sensitive cases and between the smooth and non-smooth cases, and
argue that these differences are real by showing that the
corresponding guarantees are tight.  Although we discuss the tightness
of the learning guarantees for ERM in the stochastic setting, similar
arguments can also be made for online learning for which algorithms and upper bounds are provided in the next section.

Table \ref{tab:tight} summarizes the bounds on the excess risk of the
ERM implied by Theorem \ref{thm:mainRad} as well previous bounds for Lipschitz
loss on finite-dimensional \citep{Panchenko02} and scale-sensitive
\citep{BartlettMe02} classes, and a bound for squared-loss on
finite-dimensional classes \cite[Theorem 11.7]{CesaBianchiLu06} that
can be generalized to any smooth strongly convex loss.
\begin{table}[h]
\begin{center}
  \begin{tabular}{| l | c | c |}
    \hline
     & Parametric  & Scale-Sensitive \\ 
     Loss function is:& $\text{dim}(\H)\leq d$\;\;,\;\; $\abs{h}\leq 1$ & $\Rad_n(\H) \leq \sqrt{R/n}$ \\ \hline
$D$-Lipschitz   & $\frac{d D}{n} + \sqrt{\frac{d D L^*}{n}}$ & $\sqrt{\frac{D^2 R}{n}}$ \\ \hline
$H$-smooth  & $\frac{d H}{n} + \sqrt{\frac{d H L^*}{n}}$ & $\frac{H R}{n} + \sqrt{\frac{H R L^*}{n}}$\\ \hline
$H$-smooth and $\lambda$-strongly Convex& $\frac{H}{\lambda} \frac{d H}{n}$ & $\frac{H R}{n} + \sqrt{\frac{H RL^*}{n}}$ \\
\hline
  \end{tabular}
\vspace{-0.1in}
{\small \caption{\small Bounds on the excess risk, up to
    polylogarithmic factors.}\label{tab:tight}}
\end{center}
\vspace{-0.1in}
\end{table}

We shall now show that the $1/\sqrt{n}$ dependencies in Table
\ref{tab:tight} are unavoidable.  To do so, we will consider the class
$\H = \left\{\x \mapsto \ip{\w}{\x} : \norm{\w}_2 \le 1\right\}$ of
$\ell_2$-bounded linear predictors (all norms in this Section are
Euclidean), with different loss functions, and various specific
distributions over $\X \times \Y$, where $\X = \left\{\x \in \R^d :
  \|\x\|_2 \le 1\right\}$ and $Y = [0,1]$.  For the non-parametric
lower-bounds, we will allow the dimensionality $d$ to grow with the
sample size $n$.

\noindent {\bf Infinite dimensional, Lipschitz (non-smooth), $\mathbf{L^* = 0}$ case}\\
Consider the absolute difference loss $\phi(h(\x),y) = \abs{h(\x) -
  y}$, take $d=2n$ and consider the following distribution: $X$ is
uniformly distributed over the $d$ standard basis vectors $\mbf{e}_i$
and if $X=\mbf{e}_i$, then $Y = {\scriptstyle \frac{1}{\sqrt{n}}}
r_i$, where $r_1,\ldots,r_d \in \{\pm 1\}$ is an arbitrary sequence of
signs unknown to the learner (say drawn randomly beforehand).  Taking
$\wopt = {\scriptstyle \frac{1}{\sqrt{n}}} \sum_{i=1}^n r_i
\mbf{e}_i$, $\norm{\wopt}_2=1$ and $L^*=\L{\wopt} = 0$. However any
sample $(\x_1,y_1), \ldots,(\x_n,y_n)$ reveals at most $n$ of $2n$
signs $r_i$, and no information on the remaining $\geq n$ signs.  This
means that for any algorithm used by the learner, there exists a
choice of $r_i$'s such that on at least $n$ of the remaining points
not seen by the learner the learner has to suffer a loss of at least
${\scriptstyle 1/\sqrt{n}}$, yielding an overall risk of at least
$1/(2\sqrt{n})$.

\noindent {\bf  Infinite dimensional, smooth, non-separable, even if strongly convex}\\
Consider the squared loss $\phi(h(\x),y) = (h(\x) - y)^2$ which is
$2$-smooth and $2$-strongly convex. For any $\sigma \ge 0$ let $d=\sqrt{n}/\sigma$ and consider the following distribution: $X$ is uniform over $\mbf{e}_i$ as before, but
this time $Y|X$ is random, with $Y|(X=\mbf{e}_i) \sim
\mathcal{N}(\frac{r_i}{2\sqrt{d}},\sigma)$, where again $r_i$ are
pre-determined, unknown to the learner, random signs.  The minimizer
of the expected risk is $\wopt = \sum_{i=1}^d {\scriptstyle
  \frac{r_i}{2\sqrt{d}}} \mbf{e}_i$, with
  $\norm{\wopt}={\scriptstyle\frac{1}{2}}$ and $L^*=L(\wopt)=\sigma^2$.
  Furthermore, for any $\w \in \W$,
$$
\L{\w} -\L{\wopt} = \E{\ip{\w - \wopt}{\x}}^2 =
\frac{1}{d}\sum_{i=1}^d (\w[i] - \wopt[i])^2 = \frac{1}{d}\norm{\w-\wopt}^2
$$

If the norm constraint becomes tight, i.e. $\norm{\hat{\w}}_2=1$, then
$L(\hat{\w})-L({\wopt})\geq 1/(4d)=\sigma/(4\sqrt{n}) = \sqrt{L^*}/(4 \sqrt{n})$.  Otherwise, each
coordinate is a separate mean estimation problem, with $n_i$ samples,
where $n_i$ is the number of appearances of $\mbf{e}_i$ in the
sample.  We have $\E{(\hat{\w}[i]-\wopt[i])^2}=\sigma^2/n_i$ and so
$$L(\hat{\w})-L^* = \frac{1}{d} \norm{\hat{\w}-\wopt}_2^2 = \frac{1}{d}
\sum_{i=1}^d \frac{\sigma^2}{n_i} \geq \frac{\sigma^2}{d}\frac{d^2}{\sum_i n_i} =
\frac{\sigma^2 d}{n}= \frac{\sigma}{\sqrt{n}} = \sqrt{\frac{L^*}{n}}$$

\noindent {\bf  Finite dimensional, smooth, not strongly convex, non-separable:}\\
Take $d=1$, with $X=1$ with probability $q$ and $X = 0$ with
probability $1-q$. Conditioned $X=0$ let $Y=0$ deterministically,
while conditioned on $X = 1$ let $Y=+1$ with probability $p = \frac{1}{2} +
\frac{0.2}{\sqrt{q n}}$ and $Y=-1$ with probability $1 - p$.  Consider
the following $1$-smooth loss function, which is quadratic around the correct
prediction, but linear away from it:
\begin{align*}
\phi(h(\x),y) = \begin{cases}
(h(\x)-y)^2 &\text{if $\abs{h(\x)-y}\leq 1/2$} \\
\abs{h(\x)-y}-1/4 &\text{if $\abs{h(\x)-y}\geq 1/2$}
\end{cases}
\end{align*}
First note that irrespective of choice of $\w$, when $\x=0$ and so
$y=0$ we always have $h(\x)=0$ and so suffer no loss.  This happens
with probability $1-q$.  Next observe that for $p>1/2$, the optimal
predictor is $\wopt \geq 1/2$.  However, for $n > 20$, with probability at
least $0.25$, $\sum_{i=1}^n y_i < 0$, and so the empirical minimizer
is $\hat{\w} \leq -1/2$.  We can now calculate
$$L(\hat{\w})-L^* > L(-1/2)-L(1/2) = q (2p-1) + (1 - q) 0 =  \frac{0.4\ q}{\sqrt{q n}} = \frac{0.4\ \sqrt{q}}{\sqrt{ n}}.$$
However note that for $p > 1/2$,  $\w^* = \frac{3}{2} - \frac{1}{2p}$ and so for $n > 20$, $L^*  > \frac{q}{2}$. Hence we conclude that with probability $0.25$ over the sample,
$$L(\hat{\w})-L^* > \sqrt{\frac{0.32 L^*}{n}}.$$

\section{Online and Stochastic Optimization of Smooth Convex Objectives}\label{sec:convex}

We now turn to online and stochastic convex optimization.  In these
settings, a learner chooses $\w \in \W$, where $\W$ is a closed convex
set in a normed vector space, attempting to minimize an objective
(loss) $\ell(\w,z)$ on instances $z \in \Z$, where
$\ell:\W\times\Z\rightarrow\R$ is an objective function which is
convex in $\w$.  This captures learning linear predictors using a
convex loss function $\phi(t,z)$, where $\Z = \X \times \Y$ and
$\ell(\w,(x,y)) = \phi(\ip{\w}{x},y)$, and extends well beyond
supervised learning.
 
We consider the case where the objective $\ell(\w,z)$ is $H$-smooth
w.r.t.~some norm $\norm{\w}$ (the reader may choose to think of $\W$
as a subset of an Euclidean or Hilbert space, and $\norm{\w}$ as the
$\ell_2$-norm). By this we mean that for any $z \in \Z$, and all $\w,
\w' \in \W$
$$
\norm{\nabla \ell(\w,z) - \nabla \ell(\w',z)}_* \le H \norm{\w - \w'}
$$
where $\|\cdot\|_*$ is the dual norm.  The key here is to generalize
Lemma \ref{lemma:smooth} to smoothness
w.r.t.~a vector $\w$, rather than scalar smoothness. This is done by the next
lemma.

\begin{lemma}\label{lemma:gensmooth}
For an $H$-smooth non-negative $f: \W \rightarrow
\mathbb{R}$, for all $\w \in \W$:
$$
\|\nabla f(\w)\|_* \le \sqrt{4 H f(\w)} \ .
$$
\end{lemma}
\begin{proof}
For any $\w_0$ such that $\|\w - \w_0\| \le 1$, let $g(t) = g(\w_0 + t (\w - \w_0))$.  For any $t , s \in \mathbb{R}$,
\begin{align*}
|g'(t) - g'(s)|&= |\ip{\nabla f(\w_0 + t (\w - \w_0)) - \nabla f(\w_0 + s (\w - \w_0)) }{\w - \w_0}|\\
& \le \|\nabla f(\w_0 + t (\w - \w_0)) - \nabla f(\w_0 + s (\w - \w_0))\|_* \ \|\w - \w_0\|\\
& \le H |t - s| \|\w - \w_0\|^2\\
& \le H |t-s|
\end{align*}
Hence $g$ is $H$-smooth and so by Lemma \ref{lemma:smooth}  
$
|g'(t)| \le \sqrt{4 H g(t)}
$.
Setting $t = 1$ we have, $\ip{\nabla f(\w)}{\w - \w_0} \le \sqrt{4 H f(\w)}$.  Taking supremum over $\w_0$  such that $\|\w_0 - \w\| \le 1$ we conclude that 
\begin{equation*}
\|\nabla f(\w)\|_* = \sup_{\w_0 : \|\w - \w_0\| \le 1} \ip{\nabla
  f(\w)}{\w - \w_0} \le \sqrt{4 H f(\w)} \qedhere
\end{equation*}
\end{proof}

The above lemma effectively shows that smoothness implies the so
called ``self-bounding'' property for the objective. This property is
used by \citep{Shalev07} to show optimistic type rates in the online
setting. In the following sub-section, we use the self-bounding
property implied by the above lemma along with result by
\citep{Shalev07} to obtain optimistic rates in the online setting.

In order to consider general norms, we will also need to rely on a
non-negative regularizer $F :\W \mapsto \mbb{R}$ that is a
$1$-strongly convex (see, e.g., \citep{Zalinescu02}) with respect to
the norm $\norm{\w}$ over $\W$.  For the Euclidean norm, we can use
the squared Euclidean norm regularizer: $F(\w)={\scriptstyle
  \frac{1}{2}}\norm{\w}_2^2$.

\subsection{Online Optimization Setting}

In the online convex optimization setting, we consider an $n$ round
game played between a learner and an adversary (Nature) where at each
round $i$, the player chooses a $\w_i \in \W$ and then the adversary picks a
$z_i \in \Z$.  The player's choice $\w_i$ may only depend on the
adversary's choices in {\em previous} rounds.  The goal of the player
is to have low average objective value $\frac{1}{n}\sum_{i=1}^{n}
\ell(\w_{i},z_i)$ compared to the best single choice in hind sight \citep{CesaBianchiLu06}.

A classic algorithm for this setting is Mirror Descent
\citep{BeckTe03}, which starts at some arbitrary $\w_1\in\W$ and
updates $\w_{i+1}$ according to $z_i$ and a stepsize $\eta$ (to be discussed
later) as follows:
\begin{equation}
  \label{eq:md}
  \w_{i+1} \leftarrow \arg\min_{\w\in\W} \ip{\eta\nabla
    \ell(\w_i,z_i)-\nabla F(\w_i)}{\w}+F(\w)
\end{equation}
For the Euclidean norm
with $F(\w)={\scriptstyle \frac{1}{2}}\!\norm{\w}_2^2$, the update
\eqref{eq:md} becomes projected online gradient descent \citep{Zinkevich03}:
\begin{equation}
  \label{eq:sd}
  \w_{i+1} \leftarrow \Pi_{\W}(\w_i - \eta \nabla \ell(\w_i,z_i))
\end{equation}
where $\Pi_{\W}(\w)=\arg\min_{\w'\in\W}\norm{\w-\w'}_2$ is the Euclidean
projection onto $\W$.

Equipped with Lemma \ref{lemma:gensmooth} which implies self-bounding property and a result by \citep{Shalev07} we have the following theorem that provides optimistic rates for the online learning of smooth objectives.

\begin{theorem}\label{onlinesmooth}
For any $B \in \mathbb{R}$ and $\overline{L^*}$ if we use stepsize $\eta = \frac{1}{H B^2 + \sqrt{H^2 B^4 + H B^2 n \overline{L^*}}}$ for the Mirror Descent algorithm then for any instance sequence $z_1,\ldots,\z_n \in \Z$, the average regret with respect to any
  $\w^* \in \W$ such that $F(\w^*)\leq B^2$ and $\tfrac{1}{n}\sum_{j=1}^n
  \loss(\w^*,z_i) \le \overline{L^*}$, is bounded by:
\begin{align*}
  \frac{1}{n}\sum_{i=1}^{n} \ell(\w_{i},z_i) - \frac{1}{n}\sum_{i=1}^{n}
  \ell(\w^*,z_i) \le \frac{4 H B^2}{n} + 2 \sqrt{\frac{H B^2 \overline{L^*}}{n}}  \ .
\end{align*}
\end{theorem}
Note that the stepsize depends on the bound $\overline{L^*}$ on the loss in hindsight.
\begin{proof}
  The proof follows from Lemma \ref{lemma:gensmooth} and Theorem 1 of
  \citep{Shalev07}, using $U_1 = B^2$ and $U_2 = n \overline{L^*}$ in
  the Theorem.
\end{proof}

\subsection{Stochastic Optimization I: Stochastic Mirror Descent}

An online algorithm can also serve as an efficient one-pass learning
algorithm in the stochastic setting.  Here, we again consider an
i.i.d.~sample $z_1,\ldots,z_n$ from some unknown distribution (as in
Section \ref{sec:rad}), and we would like to find $\w$ with low risk
$L(\w)=\E{\ell(\w,Z)}$.  When $z=(\x,y)$ and
$\ell(\w,z)=\phi(\ip{\w}{\x},y)$, this agrees with the supervised
learning risk discussed in the Introduction and analyzed in Section
\ref{sec:rad}.  But instead of focusing on the ERM, we run Mirror
Descent (or Projected Online Gradient Descent in case of a Euclidean
norm) on the sample, and then take $\tilde{\w}={\scriptstyle
  \frac{1}{n}}\!\sum_{i=1}^{n}\w_i$.  Standard arguments
\citep{CesaBianchiCoGe04} allow us to convert the online regret bound
of Theorem \ref{onlinesmooth} to a bound on the excess risk:
\begin{corollary}\label{cor:sgd}
For any $B \in \mathbb{R}$ and $\overline{L^*}$ if we run Mirror Descent on a random sample with stepsize $\eta =  \frac{1}{H B^2 + \sqrt{H^2 B^4 + H B^2 n \overline{L^*}}}$, then for any $\w^* \in \W$ with $F(\w^*)\leq B^2$ and $L(\w^*) \leq \overline{L^*}$, we have
$$\E{\L{\tilde{\w}}} - \L{\wopt} \leq \frac{4 H B^2}{n} + 2 \sqrt{\frac{H B^2 \overline{L^*}}{n}} ,
$$
where the expectation is over the sample.
\end{corollary}
Again, one must know a bound $\overline{L^*}$ on the risk in order to
choose the stepsize. 

It is instructive to contrast this guarantee with similar looking
guarantees derived recently in the stochastic convex optimization
literature \citep{Lan09}.  There, the model is stochastic first-order
optimization, i.e. the learner gets to see an unbiased estimate
$\nabla l(\w,z_i)$ of the gradient of $L(\w)$. The variance of the
estimate is assumed to be bounded by $\sigma^2$. The expected accuracy
after $n$ gradient evaluations then has two terms: a ``accelerated''
term that is $O(H/n^2)$ and a slow $O(\sigma/\sqrt{n})$ term. While
this result is applicable more generally (since it does not require
non-negativity of $\ell$), it is not immediately clear if our
guarantees can be derived using it. The main difficulty is that
$\sigma$ depends on the norm of the gradient estimates. Thus, it
cannot be bounded in advance even if we know that $L(\wopt)$ is small.
That said, it is intuitively clear that towards the end of the
optimization process, the gradient norms will typically be small if
$L(\wopt)$ is small because of the self bounding property
(Lemma~\ref{lemma:gensmooth}).  Exploring this connection can be
fruitful direction for further research.

\subsection{Stochastic Optimization II: Regularized Batch Optimization}

It is interesting to note that using stability arguments, a guarantee
very similar to Corollary \ref{cor:sgd}, avoiding the polylogarithmic
factors of Theorem \ref{thm:mainRad} as well as the dependence on the
bound on the loss ($b$ in Theorem \ref{thm:mainRad}), can be obtained
also for a ``batch'' learning rule similar to ERM, but incorporating
penalty-type regularization.  For a given regularization parameter
$\lambda > 0$ define the regularized empirical loss as
$$
\hL_\lambda(\w) := \hL(\w) + \lambda F(\w)
$$
and consider the Regularized Empirical Risk Minimizer
\begin{equation}\label{eq:regerm}
\hat{\w}_{\lambda} = \arg\min_{\w\in\W} \hL_\lambda(\w)
\end{equation}
The following theorem provides a bound on excess risk similar to Corollary \ref{cor:sgd}:

\begin{theorem}\label{thm:regerm}
For any $B \in \mathbb{R}$ and $\overline{L^*}$ if we set $\lambda = \frac{128 H}{n} + \sqrt{\frac{128^2 H^2}{n^2} + \frac{128 H \overline{L^*}}{n B^2}}$
then for all $\wopt \in \W$ with $F(\wopt)\leq B^2$ and $L(\wopt)\leq
 \overline{L^*}$,  we have
\begin{equation*}
\E{\L{\hat{\w}_{\lambda}}} - \L{\wopt} \leq \frac{256 H B^2}{n} + \sqrt{\frac{2048 H B^2 \overline{L^*}}{n}} \ ,
\end{equation*}
where the expectation is over the sample of size $n$.
\end{theorem}

To prove Theorem \ref{thm:regerm}, we use stability arguments similar
to the ones used by \citep{ShalevShSrSr09}, which
are in turn based on \citep{BousquetEl02}.
However, while \citep{ShalevShSrSr09} use the
notion of uniform stability, here it is necessary to look at stability
in expectation to get the faster rates (uniform stability does not
hold with the desired rate).

To use stability based arguments, for each $i \in [n]$ we consider a
perturbed sample where instance $z_i$ is replaced by instance $z'_i$
drawn independently from same distribution as $z_i$.  Let
$\hL^{(i)}(\w) = {\scriptsize \frac{1}{n}}(\sum_{j\not=i}
\ell(\w,z_j)+\ell(\w,z'_i))$ be the empirical risk over the perturbed
sample, and consider the corresponding regularized empirical risk
minimizer $\hw^{(i)}_\lambda = \arg\min_{\w} \hL^{(i)}_\lambda(\w)$,
where $\hL^{(i)}_\lambda(\w) = \hL^{(i)}(\w) +\lambda F(\w)$.  We
first prove the following lemma on the expected stability of the
regularized minimizer.

\begin{lemma}\label{lem:stability}
For any $i \in [n]$ we have that
$$
\Es{z_1,\ldots,z_n, z'_i}{\ell(\hw^{(i)}_\lambda,z_i) - \ell(\hw_\lambda,z_i) } \le \frac{32 H}{\lambda n} \Es{z_1,\ldots,z_n}{L(\hw_\lambda)} \ .
$$
\end{lemma}
\begin{proof}
\begin{align*}
\hL_\lambda(\hw^{(i)}_\lambda) - \hL_\lambda(\hw_\lambda) & = \frac{\ell(\hw^{(i)}_\lambda,z_i) - \ell(\hw_\lambda,z_i) }{n} + \frac{\ell(\hw_\lambda,z_i') - \ell(\hw^{(i)}_\lambda,z_i') }{n} + \hL^{(i)}_\lambda(\hw^{(i)}_\lambda) - \hL^{(i)}_\lambda(\hw_\lambda)  \\
& \le \frac{\ell(\hw^{(i)}_\lambda,z_i) - \ell(\hw_\lambda,z_i) }{n} + \frac{\ell(\hw_\lambda,z_i') - \ell(\hw^{(i)}_\lambda,z_i') }{n} \\
& \le \frac{1}{n} \|\hw^{(i)}_\lambda - \hw_\lambda \| \left(\|\nabla \ell(\hw^{(i)}_\lambda,z_i)\|_* + \|\nabla \ell(\hw_\lambda,z_i')\|_*\right) \\
& \le \frac{2 \sqrt{H}}{n} \|\hw^{(i)}_\lambda - \hw_\lambda \| \left( \sqrt{\ell(\hw^{(i)}_\lambda,z_i)} + \sqrt{\ell(\hw_\lambda,z_i')}\right) 
\end{align*}
where the last inequality follows from Lemma \ref{lemma:gensmooth}. By $\lambda$-strong convexity of $\hL_\lambda$ we have that
$$
\hL_\lambda(\hw^{(i)}_\lambda) - \hL_\lambda(\hw_\lambda) \ge \frac{\lambda}{2} \|\hw^{(i)}_\lambda - \hw_\lambda \|^2.
$$
We can conclude that
$$
\|\hw^{(i)}_\lambda - \hw_\lambda \| \le \frac{4 \sqrt{H}}{\lambda n}  \left( \sqrt{\ell(\hw^{(i)}_\lambda,z_i)} + \sqrt{\ell(\hw_\lambda,z_i')}\right)
$$
This gives us:
\begin{align*}
\ell(\hw^{(i)}_\lambda,z_i) - \ell(\hw_\lambda,z_i)  &\le \|\nabla \ell(\hw^{(i)}_\lambda,z_i)\|_* \|\hw^{(i)}_\lambda - \hw_\lambda \| \\
& \le \sqrt{4 H \ell(\hw^{(i)}_\lambda,z_i)} \left(\frac{4 \sqrt{H}}{\lambda}  \left( \sqrt{\ell(\hw^{(i)}_\lambda,z_i)} + \sqrt{\ell(\hw_\lambda,z_i')}\right) \right)\\
& \le \frac{16 H}{\lambda n} \left( \ell(\hw^{(i)}_\lambda,z_i) + \ell(\hw_\lambda,z_i')\right)
\end{align*}
Taking expectation:
\begin{multline}
\Es{z_1,\ldots,z_n, z'_i}{\ell(\hw^{(i)}_\lambda,z_i) - \ell(\hw_\lambda,z_i)}  \le \frac{16 H}{\lambda n} \Es{z_1,\ldots,z_n, z'_i}{ \ell(\hw^{(i)}_\lambda,z_i) + \ell(\hw_\lambda,z_i')} \\
= \frac{16 H}{\lambda n} \Es{z_1,\ldots,z_n, z'_i}{
  \L{\hw^{(i)}_\lambda}+ \L{\hw_\lambda}} = \frac{32 H}{\lambda n}
\Es{z_1,\ldots,z_n}{ \L{\hw_\lambda}} \notag\qedhere
\end{multline}

\end{proof}

\begin{proof}[Proof of Theorem \ref{thm:regerm}]
By Lemma \ref{lem:stability} we have :
\begin{align*}
\Es{z_1,\ldots,z_n}{L_\lambda(\hw_\lambda) - L_\lambda(\wopt_\lambda)} & \le \Es{z_1,\ldots,z_n}{L_\lambda(\hw_\lambda) - \hL_\lambda(\hw_\lambda)} = \Es{z_1,\ldots,z_n}{L(\hw_\lambda) - \hL(\hw_\lambda)}\\
& = \frac{1}{n} \sum_{i=1}^n \Es{z_1,\ldots,z_n, z'_i}{\ell(\hw^{(i)}_\lambda,z_i) - \ell(\hw_\lambda,z_i) } \le \frac{32 H}{\lambda n} \Es{z_1,\ldots,z_n}{L(\hw_\lambda)}
\end{align*}
Noting the definition of $\hL_\lambda(\w)$ and rearranging we get
\begin{align*}
\Es{z_1,\ldots,z_n}{L(\hw_\lambda) - L(\wopt)} & \le \frac{32 H}{\lambda n} \Es{z_1,\ldots,z_n}{L(\hw_\lambda)} + \lambda F(\wopt) - \lambda F(\hw_\lambda) \le \frac{32 H}{\lambda n} \Es{z_1,\ldots,z_n}{L(\hw_\lambda)} + \lambda F(\wopt) 
\end{align*}
Rearranging further we get
$$
\Es{z_1,\ldots,z_n}{\L{\hw_\lambda}} - \L{\wopt} \le \left(\frac{1}{1 - \frac{32 H}{\lambda n}} - 1\right) \L{\wopt} + \frac{\lambda}{1 - \frac{32 H}{\lambda n}} F(\wopt)
$$
plugging in the value of $\lambda$ gives the result.
\end{proof}

\section{Implications}

We demonstrate the implications of our results in several settings.

\subsection{Improved Margin Bounds}

``Margin bounds'' provide a bound on the expected zero-one loss of a
classifiers based on the margin zero-one error on the training sample.
\citep{KoltchinskiiPa02} provides margin
bounds for a generic class $\H$ based on the Rademacher complexity of
the class.  This is done by using a non-smooth Lipschitz ``ramp'' loss
that upper bounds the zero-one loss and is upper-bounded by the margin
zero-one loss.  However, such an analysis unavoidably leads to a
$1/\sqrt{n}$ rate even in the separable case, since as we discuss in
Section \ref{sec:tight}, it is not possible to get a faster rate for a
non-smooth loss.  Following the same idea we use the following smooth
``ramp'':
\begin{align*} \phi(t) = \begin{cases}
      1 & t \le 0 \\
      \frac{1 + \mathrm{cos}(\pi t/\gamma)}{2} & 0 < t < \gamma \\
      0 & t \ge \gamma
	\end{cases} \ .
\end{align*}
This loss function is $\frac{\pi^2}{4 \gamma^2}$-smooth and is lower
bounded by the zero-one loss and upper bounded by the $\gamma$ margin
loss. Using Theorem \ref{thm:mainRad}, we can now provide improved
margin bounds for the zero-one loss of any classifier based on
empirical margin error. Let
$$\err(h) = \E{\ind{h(x) \ne y}}$$
be the
zero-one risk and, for any $\gamma >
0$ and sample $(\x_1,y_1),\ldots,(\x_n,y_n) \in \X \times \{\pm1\}$,
define the $\gamma$-margin empirical zero one risk as
$$
\herr_\gamma(h) := \frac{1}{n} \sum_{i=1}^n \ind{y_i h(\x_i) < \gamma} \ .
$$
\begin{theorem}\label{thm:margin}
  For any hypothesis class $\H$, with $\abs{h}\leq b$, and any $\delta
  > 0$, with probability at least $1 - \delta$, simultaneously for all
  margins $\gamma > 0$ and all $h \in\H $:
\begin{align*}
\err(h) \le \herr_\gamma(h) + K\left( \sqrt{\herr_\gamma(h)} \left(\tfrac{\log^{1.5} n }{\gamma}\hRad_n(\H)  + \sqrt{\tfrac{ \log(\log(\frac{4b}{\gamma})/\delta)}{n}}\right)+ \tfrac{\log^{3} n }{\gamma^2} \hRad_n^2(\H)  + \tfrac{\log(\log(\frac{4 b}{\gamma})/\delta)}{n}\right)
\end{align*}
where $K$ is a numeric constant from Theorem \ref{thm:mainRad} 
\end{theorem}
In particular, the above bound implies:
\begin{align*}
\err(h) \le 1.01\, \herr_\gamma(h) + K\left( \frac{ \log^{3} n }{\gamma^2} \hRad_n^2(\H)  + \frac{ \log(\log(\frac{4 b}{\gamma})/\delta)}{n} \right)
\end{align*}
where $K$ is an appropriate numeric constant.

Improved margin bounds of the above form have been previously shown
specifically for linear prediction in a Hilbert space (as in Support
Vector Machines) based on the PAC Bayes theorem
\citep{McAllester03,LangfordSh03}.  However these PAC-Bayes based
results are specific to the linear function class.  Theorem
\ref{thm:margin} is, in contrast, a generic concentration-based result
that can be applied to any function class with and yields rates dominated
by $\hRad^2(\H)$.

\subsection{Interaction of Norm and Dimension}

Consider the problem of learning a low-norm linear predictor with
respect to the squared loss $\phi(t,z)=(t-z)^2$, where $\X \in \R^d$,
for finite but very large $d$, and where the expected norm of $X$ is low.
Specifically, let $X$ be Gaussian with $\E{\norm{X}^2}=B$,
$Y=\ip{\w^*}{X}+\mathcal{N}(0,\sigma^2)$ with $\norm{\w^*}=1$, and
consider learning a linear predictor using $\ell_2$ regularization.  What
determines the sample complexity?  How does the error decrease as the
sample size increases?

From a scale-sensitive statistical learning perspective, we expect
that the sample complexity, and the decrease of the error, should
depend on the norm $B$, especially if $d \gg B^2$.  However, for any
fixed $d$ and $B$, even if $d \gg B^2$, asymptotically as the number of
samples increase, the excess risk of norm-constrained or
norm-regularized regression actually behaves as $L(\hat{\w}) - L^*
\approx \frac{d}{n} \sigma^2$, and depends (to first order) only on
the dimensionality $d$ and not at all on $B$ \citep{LiangBaBoJo10}.
How does the scale sensitive complexity come into play?

The asymptotic dependence on the dimensionality alone can be
understood through Table \ref{tab:tight}.  In this non-separable situation,
parametric complexity controls can lead to a $1/n$ rate, ultimately
dominating the $1/\sqrt{n}$ rate resulting from $L^*>0$ when
considering the scale-sensitive, non-parametric complexity control
$B$.  (The dimension-dependent behavior here is actually a bit better
then in the generic situation---the well-posed Gaussian model allows
the bound to depend on $\sigma^2=L^*$ rather then on
$\sup(\ip{\w}{\x}-y)^2 \approx B^2+\sigma^2$).

Combining Theorem \ref{thm:regerm} with the asymptotic $\frac{d}{n}\sigma^2$
behavior, and noting that at the worst case we can predict using a
zero vector, yields the following overall picture on the expected
excess risk of ridge regression with an optimally chosen $\lambda$:
$$
L(\hat{\w}_{\lambda})-L^* \leq O \left( \min\left( B^2 , \frac{B^2}{n} +  \frac{B\sigma}{\sqrt{n}} , \frac{d \sigma^2}{n} \right) \right)
$$
Roughly speaking, each term above describes the
behavior in a different regime of the sample size: 
\begin{itemize}
\item The first (``random'') regime until $n = \Theta(B^2)$ where the
  excess risk is $B^2$.
\item The second (``low-noise'') regime, where the excess risk is dominated
  by the norm and behaves as $B^2/n$, until $n = \Theta(B^2/\sigma^2)$
  and $L(\hat{\w}) = \Theta(L^*)$.
\item The third (``slow'') regime, where the excess risk is controlled
  by the norm and the approximation error and behaves as
  $B\sigma/\sqrt{n}$, until $n = \Theta(d^2 \sigma^2/B^2)$ and
  $L(\hat{\w})=L^*+\Theta(B^2/d)$.
\item the fourth (``asymptotic'') regime, where the excess risk is
  dominated by the dimensionality and behaves as $d/n$.
\end{itemize}
This sheds further light on recent work on this phenomena by Liang and
Srebro based on exact asymptotics of simplified situations
\citep{LiangSr10}.

\subsection{Sparse Prediction}

The use of the $\ell_1$ norm has become very popular for learning
sparse predictors in high dimensions, as in the LASSO.  The LASSO
estimator \citep{Tibshirani96b} $\hat{\w}$ is obtained by considering
the squared loss $\phi(z,y)=(z-y)^2$ and minimizing $\hat{L}(\w)$
subject to $\|\w\|_1 \le B$.  Let us assume there is some (unknown)
sparse reference predictor $\w^0$ that has low expected loss and
sparsity (number of non-zeros) $\norm{\w^0}_0=k$, and that
$\|\x\|_\infty \le 1, y\le 1$.  In order to choose $B$ and apply
Theorem~\ref{thm:mainRad} in this setting, we need to bound
$\norm{\w^0}_1$.  This can be done by, e.g., assuming that the
features $\x[i]$ {\em in the support of $\w^0$} are mutually
uncorrelated.  Under such an assumption, we have: $\norm{\w^0}_1^2 \le
k\E{\ip{\w^0}{x}^2} \leq 2k(L(\w^0)+\E{y^2}) \le 4k$. Thus,
Theorem~\ref{thm:mainRad} along with Rademacher complexity bounds from
\citep{KakadeSrTe08} gives us,
\begin{equation}\label{eq:LASSObound}
L(\hat{\w}) \le L(\w^0) + \widetilde{O}\left(
\frac{k\,\log(d)}{n} + \sqrt{\frac{ k\,L(\w^0)\,\log(d)}{n}} \right).
\end{equation}
It is possible to relax the no-correlation assumption to a bound on
the correlations, as in mutual incoherence, or to other weaker
conditions \citep{ShalevSrZh09}.  But, in any case, unlike typical analysis for
compressed sensing, where the goal is recovering $\w^0$ itself, here
we are only concerned with correlations {\em inside the support of
  $\w^0$}.  Furthermore, we do not need to require that the optimal
predictor is sparse or close to being sparse, or that the model is
well specified: only that there exists a good (low risk) predictor
using a small number of fairly uncorrelated features.

Bounds similar to \eqref{eq:LASSObound} have been derived using
specialized arguments
\citep{Koltchinskii09,VanDeGeer08,bickel2009simultaneous}---here we
demonstrate that a simple form of these bounds can be obtained under
very simple conditions, using the generic framework we suggest.

It is also interesting to note that the methods and results of Section
\ref{sec:convex} can also be applied to this setting.  But since
$\norm{\w}^2_1$ is {\em not} strongly convex with respect to
$\norm{\w}_1$, we must instead use the entropy regularizer
\begin{equation}
  \label{eq:Fent}
  F(\w) = B \sum_i \x[i] \log \left(\frac{\x[i]}{1/d}\right) + \frac{B^2}{e}
\end{equation}
which is  1-strongly convex w.r.t.
$\norm{\cdot}_1$ on $\W = \left\{ \w \in \R^d \middle| \w[i]\geq 0,
  \norm{\w}_1 \leq B \right\}$, with $F(\w) \leq B^2 (1+\log d)$ (we
consider here only non-negative weights---in order to allow $\w[i]<0$
we can include also each features negation, doubling the
dimensionality).  Recalling that $\norm{\w^0}_1 \leq 2\sqrt{k}$ and
using $B=2\sqrt{k}$ in \eqref{eq:Fent}, we have from Theorem
\ref{thm:regerm} we that:
\begin{equation}\label{eq:entbound}
L(\hat{\w}_{\lambda}) \le L(\w^0) + O\left(
\frac{k\,\log(d)}{n} + \sqrt{\frac{ k\,L(\w^0)\,\log(d)}{n}} \right).
\end{equation}
where $\hat{\w}_{\lambda}$ is the regularized empirical minimizer
\eqref{eq:regerm} using the entropy regularizer \eqref{eq:Fent} with
$\lambda$ as in Theorem \ref{thm:regerm}.  The advantage here is that
using Theorem \ref{thm:regerm} instead of Theorem \ref{thm:mainRad}
avoids the extra logarithmic factors (yielding a clean big-$O$
dependence in \eqref{eq:entbound} as opposed to big-$\widetilde{O}$ in
\eqref{eq:LASSObound}).

More interestingly, following Corollary \ref{cor:sgd}, one can use
stochastic mirror descent, taking steps of the form \eqref{eq:md} with
the entropy regularizer \eqref{eq:Fent}, to obtain the same performance
guarantee as inn \eqref{eq:entbound}.  This provides an efficient,
single-pass optimization approach to sparse prediction as an
alternative to batch optimization with an $\ell_1$-norm constraint,
and yielding the same (if not somewhat better) guarantees.

\section{Discussion}
We use the term ``optimistic rates'' as opposed to ``fast rates'' to
distinguish between the rates of the form we get in equation
\eqref{eq:intromain} from the ones where $L(\hat{h}) - L^*$ is bounded
only by $O(H R/n)$. Of course when $L^*$ is smaller than
$\frac{R}{n}$ then one can obtain a bound of $O(H R/n)$ for
$L(\hat{h}) - L^*$ using the optimistic rates. However, in general, for
optimistic rates one has an extra $\sqrt{L^* H R/n}$ term in the rate
as compared to fast rates.  While there is this crucial distinction
between ``optimistic'' and ``fast'' rates, we would like to
point out that the bound~\ref{eq:intromain} can be
re-written for any $a > 0$ as,
$$
L(\hat{h})  \le (1 + a) L^* + \widetilde{O}\left(\left(1 + \frac{1}{a} \right)\frac{H R }{n}\right)
$$
As an example taking $a = 0.01$ this implies that $L(\hat{h}) - 1.01\, L^*$ converges as $H R/n$. Hence in practice especially since one tries to pick $\H$ so that $L^*$ is small, the optimistic bounds implies fast learning rates. 
 
The notion of Rademacher complexity used throughout this work is that
of worst-case Rademacher complexity, that is supremum over sample of
size $n$.  With a Lipschitz loss, it is possible to obtain guarantees
similar to \eqref{eq:lip} also in terms of the expected Rademacher
complexity (taking an expectation over samples of size $n$), or even
the empirical Rademacher complexity, calculated only on the specific
sample observed \citep{BartlettMe02}.  A natural question is whether the
worst case Rademacher complexity used in Theorem \ref{thm:mainRad} can
be replaced by the expected Rademacher complexity.  The difference
between worst case and expected Rademacher complexities might be
crucial in certain applications. For example,
\citep{FoySre11} use our Theorem \ref{thm:mainRad} to obtain guarantees
on matrix completion with max-norm regularization under any arbitrary
distribution.  While this approach gave meaningful rates for matrix
completion with max-norm, the Rademacher complexity of a trace-norm
constrained class can only be meaningfully bounded {\em on average}.  

Unfortunately, such a generalization is {\em not} possible: as \citep{FoySre11}
show, it is not possible to meaningfully generalize with respect to the
squared loss by constraining the trace-norm, even with a uniform
distribution where the expected Rademacher complexity is nicely
behaved.  This shows that our Theorem \ref{thm:mainRad} {\em cannot}
be restated in terms of the expected or empirical Rademacher
complexity, in sharp contrast to the case of Lipschitz bounded loss.
An interesting question is what happens when the loss function is
Lipschitz {\em and} smooth (e.g. the logistic loss or smoothed hinge
loss).  Of course, in such cases a guarantee of the form \eqref{eq:lip}
can be obtained in terms of the expected Rademacher complexity,
replying only on the Lipschitz constant of the loss function.  But we
suspect that if the loss is Lipschitz and smooth (or bounded and
smooth), it is also possible to obtain an optimistic rate similar to
\eqref{eq:intromain} in terms of the expected or empirical Rademacher
complexity.

\bibliographystyle{plain}
\bibliography{bib}

\include{appendix}

\end{document}

%% file: header.tex
\usepackage{amsmath,amsfonts,amssymb,amsthm}
\usepackage{graphicx}
\usepackage{color}
\usepackage{tikz}
\usetikzlibrary{arrows,shapes}
\usetikzlibrary{patterns}

 \newcommand{\ind}[1]{\ 1\hspace{-2.3mm}{1} _{\{#1\}}}

\newtheorem{theorem}{Theorem}
\newtheorem{lemma}{Lemma}[section]
\newtheorem{corollary}[theorem]{Corollary}
\newtheorem*{remark}{Remark}
\newtheorem{proposition}[theorem]{Proposition}

\newcommand{\reals}{\ensuremath{\mathbb{R}}}
\newcommand{\Z}{\ensuremath{\mathbb{Z}}}
\newcommand{\R}{\ensuremath{\mathbb{R}}}

\newcommand{\argmin}[1]{\underset{#1}{\mathrm{argmin}} \:}

\newcommand{\E}[1]{{\mathbb{E}\left[{#1}\right]}}

\newcommand{\hE}[1]{{\mathbb{\hat{E}}\left[{#1}\right]}}
\newcommand{\Ep}[2]{{\mathbb{E}_{#1}\left[{#2}\right]}}

\newcommand{\ip}[2]{{\left\langle{#1},{#2}\right\rangle}}

\newcommand{\abs}[1]{\left\lvert{#1}\right\rvert}
\newcommand{\norm}[1]{\left\lVert{#1}\right\rVert}

%% file: notation.tex
\newcommand{\z}{z}

\newcommand{\w}{\ensuremath{\mathbf{w}}}

\renewcommand{\v}{\ensuremath{\mathbf{v}}}

\newcommand{\W}{\ensuremath{\mathbf{W}}}

\newcommand{\x}{\ensuremath{\mathbf{x}}}
\newcommand{\loss}{\ell}

\newcommand{\wopt}{\w^{\star}} 

\newcommand{\F}{{F}}

\newcommand{\X}{\mathcal{X}}
\newcommand{\Y}{\mathcal{Y}}

\newcommand{\Rad}{R}

%% file: appendix.tex
\appendix

\section{Relating Covering Numbers, Fat Shattering Dimension, and
  Rademacher Complexity}\label{app:rel}

Recall that the proof of the main lemma (Lemma \ref{lem:key}) relies
on moving between various complexity measures.  To this end, we state
and prove bounds on the relationship between these complexity
measures, namely covering numbers, fat-shattering dimensions and the
Rademacher complexity, some of which might be of independent interest.
These bounds extend and refine previously existing results, but for completeness we provide full proofs for all
the bounds used.  Before we proceed, recall the following definitions
of covering numbers and fat shattering dimension.  For any $\epsilon
>0$ and function class $\mathcal{F} \subset \mathbb{R}^{\mathcal{Z}}$:
\begin{description}
\item[ ] \vspace{-0.1in} The $L_2$ covering number $\mathcal{N}_2\left(\mathcal{F},\epsilon, n\right)$ is the supremum over samples $z_1,\ldots,z_n$ of the size of a minimal cover $\mathcal{C}_\epsilon$ such that 
$\forall f \in \mathcal{F}$, $\exists f_\epsilon \in
\mathcal{C}_\epsilon$ s.t.~$\sqrt{\frac{1}{n}
  \sum_{i=1}^n (f(z_i) - f_\epsilon(z_i))^2} \le
\epsilon$.\\
\item[ ] \vspace{-0.15in} The $L_\infty$ covering number $\mathcal{N}_\infty\left(\mathcal{F},\epsilon, n\right)$  is the supremum over samples $z_1,\ldots,z_n$ of the size of a minimal cover $\mathcal{C}_\epsilon$ such that
$\forall f \in \mathcal{F}$,  $\exists f_\epsilon \in
\mathcal{C}_\epsilon$ s.t.~$ \max_{i \in [n]} |f(z_i) - f_\epsilon(z_i)| \le
\epsilon$.\\
\item[ ] \vspace{-0.15in} The fat-shattering dimension
  $\mathrm{fat}_{\epsilon}(\mathcal{F})$ at scale $\epsilon$ is the
  maximum number of points $\epsilon$-shattered by $\mathcal{F}$ (see
  e.g.~\citep{Mendelson02}), that is largest $d \in \mathbb{N}$ such that there exists $d$ points, $x_1,\ldots,x_d \in \X$ and witnesses $s_1,\ldots, s_d \in \reals$ such that,
  $$
\forall \sigma_1,\ldots,\sigma_d \in \{\pm1\}, \exists f \in \F \textrm{ s.t. }\forall i \in [d],~ \sigma_i (f(x_i) - s_i) \ge \epsilon/2
  $$
\end{description}

We present bounds on the Rademacher complexity in terms of the $L_2$
covering numbers (Lemma~\ref{lem:cover}), on the $L_{\infty}$ covering numbers in
terms of the fat shattering dimension (Lemma~\ref{lem:covfat}), and then on the
fat-shattering dimension back in terms of the worst-case Rademacher
complexity (Lemma~\ref{lem:fatrad}).

\subsection{The Refined Dudley Integral: Bounding Rademacher
  Complexity with $L_2$ Covering Numbers}

We shall find it simpler here to use the empirical Rademacher
complexity for a given sample $x_1,\ldots,x_n$ \citep{BartlettMe02}:
\begin{equation}
  \label{eq:empraddef}
\hat{R}_n(\H) =  \Es{\sigma \sim \text{Unif}(\{\pm 1\}^n)}{\sup_{h \in \H} \frac{1}{n} \left|\sum_{i=1}^n  h(x_i) \sigma_i \right|}
\end{equation}
and the $L_2$ covering number at scale $\epsilon > 0$ specific to a sample $x_1,\ldots,x_n$, denoted by $N_2\left(\epsilon,\mathcal{F},(x_1,\ldots,x_n)\right)$ as the size of a minimal cover $\mathcal{C}_\epsilon$ such that 
$$\forall f \in \mathcal{F}, \exists f_\epsilon \in
\mathcal{C}_\epsilon\ \textrm{s.t.}~\sqrt{\frac{1}{n}
  \sum_{i=1}^n (f(z_i) - f_\epsilon(z_i))^2} \le
\epsilon~.$$ We will also denote $\hE{f^2} = \frac{1}{n} \sum_{i=1}^n f^2(x_i)$.

We state our bound in terms of the empirical Rademacher
complexity and covering numbers.  Taking a supremum over samples of
size $n$, we get the same relationship between the worst-case
Rademacher complexity and covering numbers, as is used in Section
\ref{sec:rad}.

The below lemma relating Empirical Rademacher complexity and covering numbers is based on refinements of the well-known Dudley Integral \citep{Dudley78}. The refinements provided in the below lemma use ideas from \citep{Koltchinski05} and from \citep{Mendelson02}.

\begin{lemma}\label{lem:cover}
For any function class $\mathcal{F}$ containing functions $f : \mathcal{X} \mapsto \mathbb{R}$, we have that
$$
\hat{R}_n(\mathcal{F}) \le \inf_{\alpha \ge 0}\left\{ 4 \alpha + 10 \int_{\alpha}^{\sup_{f \in \F} \sqrt{\hE{f^2}}} \sqrt{\frac{\log \mathcal{N}_2\left(\epsilon,\mathcal{F},(x_1,\ldots,x_n)\right)}{n}} d \epsilon \right\} \ .
$$
\end{lemma}
\begin{proof}
Let $\beta_0 = \sup_{f \in \F}\sqrt{ \hE{f^2}} $ and for any $j \in \mathbb{Z}_+$ let $\beta_j = 2^{-j} \sup_{f \in \F} \sqrt{\hE{f^2}}$. The basic trick here is the idea of chaining. For each $j$ let $T_i$ be a (proper) $L_2$-cover  at scale $\beta_j$ of $\mathcal{F}$ for the given sample. For each $f \in \mathcal{F}$ and $j$, pick an $\hat{f}_i \in T_i$ such that $\hat{f}_i$ is an $\beta_i$ approximation of $f$. Now for any $N$, we express $f$ by chaining as
$$
f = f - \hat{f}_N + \sum_{i=1}^{N}\left(\hat{f}_{i} - \hat{f}_{i-1}\right)
$$
where $\hat{f}_0 = 0$. Hence for any $N$ we have that
\begin{align}
\hat{R}_n(\mathcal{F}) &= \frac{1}{n}\Ep{\sigma}{\sup_{f \in \mathcal{F}} \sum_{i=1}^n  \sigma_i \left(f(\x_i) - \hat{f}_N(\x_i) + \sum_{j=1}^N \left( \hat{f}_j(\x_i) - \hat{f}_{j-1}(\x_i)\right) \right)} \notag \\
& \le \frac{1}{n}\Ep{\sigma}{\sup_{f \in \mathcal{F}} \sum_{i=1}^n  \sigma_i \left(f(\x_i) - \hat{f}_N(\x_i) \right)} + \sum_{j=1}^N \frac{1}{n} \Ep{\sigma}{\sup_{f \in \mathcal{F}} \sum_{i=1}^n \sigma_i \left( \hat{f}_j(\x_i) - \hat{f}_{j-1}(\x_i)\right)} \notag \\
& \le \frac{1}{n} \sqrt{\sum_{i=1}^n \sigma_i^2}\ \sup_{f \in \mathcal{F}} \sqrt{ \sum_{i=1}^n (f(x_i) - \hat{f}_N(x_i)^2} + \sum_{j=1}^N \frac{1}{n} \Ep{\sigma}{\sup_{f \in \mathcal{F}} \sum_{i=1}^n \sigma_i \left( \hat{f}_j(\x_i) - \hat{f}_{j-1}(\x_i)\right)} \notag\\
& \le \beta_N +  \sum_{j=1}^N \frac{1}{n} \Ep{\sigma}{\sup_{f \in \mathcal{F}} \sum_{i=1}^n \sigma_i \left( \hat{f}_j(\x_i) - \hat{f}_{j-1}(\x_i)\right)} \label{eq:chain}
\end{align}
where the step before last is due to Cauchy-Shwarz inequality and $\mathbf{\sigma} = \left[ \sigma_1, ...,\sigma_n \right]^\top$. Now note that 
\begin{align*}
\frac{1}{n} \sum_{i=1}^n (\hat{f}_j(x_i) - \hat{f}_{j-1}(x_i))^2 & = \frac{1}{n} \sum_{i=1}^n \left( (\hat{f}_j(x_i)) - f(x_i)) + (f(x_i) - \hat{f}_{j-1}(x_i)) \right)^2 \\
& \le \frac{2}{n} \sum_{i=1}^n \left( \hat{f}_j(x_i)) - f(x_i)\right)^2 +  \frac{2}{n} \sum_{i=1}^n \left(f(x_i) - \hat{f}_{j-1}(x_i) \right)^2 \\
& \le 2 \beta_j^2 + 2 \beta_{j-1}^2 = 6 \beta_j^2 \ .
\end{align*}
Now Massart's finite class lemma \citep{Massart00} states that if for any function class $\mathcal{G}$, $\sup_{g \in \mathcal{G}} \sqrt{\frac{1}{n} \sum_{i=1}^n g(x_i)^2 } \le R$, then 
$\hat{R}_{n}(\mathcal{G}) \le \sqrt{\frac{2 R^2 \log(|\mathcal{G}|)}{n}}$. Applying this to function classes $\{f - f' : f \in T_j ,\ f' \in T_{j-1}\}$ (for each $j$) we get from \eqref{eq:chain} that for any $N$,
\begin{align*}
\hat{R}_n(\mathcal{F}) & \le \beta_N +  \sum_{j=1}^N  \beta_j \sqrt{\frac{12 \log(|T_j|\ |T_{j-1}|)}{n}} \\
& \le \beta_N +   \sum_{j=1}^N  \beta_j \sqrt{\frac{24 \ \log\ |T_j|}{n}} \\
& \le \beta_N +  10 \sum_{j=1}^N  (\beta_j - \beta_{j+1}) \sqrt{\frac{\log\ |T_j|}{n}} \\
& \le \beta_N +  10 \sum_{j=1}^N  (\beta_j - \beta_{j+1}) \sqrt{\frac{\log\ \mathcal{N}_2\left(\beta_j,\mathcal{F},(x_1,\ldots,x_n)\right)}{n}} \\
& \le \beta_N +  10 \int_{\beta_{N+1}}^{\beta_{0}} \sqrt{\frac{\log\ \mathcal{N}_2\left(\epsilon,\mathcal{F},(x_1,\ldots,x_n)\right)}{n}} d \epsilon
\end{align*}
where the third step is because $2 (\beta_j - \beta_{j+1}) = \beta_j$ and we bounded $\sqrt{24}$ by $5$. Now for any $\alpha > 0$, pick $N = \sup\{j : \beta_j > 2 \alpha\}$. In this case we see that by our choice of $N$, $\beta_{N+1} \le 2 \alpha$ and so $\beta_N = 2 \beta_{N+1} \le 4 \epsilon$. Also note that since $\beta_{N} > 2 \alpha$, $\beta_{N+1} = \frac{\beta_N}{2}> \alpha$. Hence we conclude that
 \begin{align*}
\hat{R}_n(\mathcal{F}) & \le 4 \alpha +  10 \int_{\alpha}^{\sup_{f \in \F} \sqrt{\hE{f^2}}} \sqrt{\frac{\log\ \mathcal{N}_2\left(\epsilon,\mathcal{F},(x_1,\ldots,x_n)\right)}{n}} d \epsilon \ .
\end{align*}
Since the choice of $\alpha$ was arbitrary we take an infimum over $\alpha$.
\end{proof}

\subsection{Bounding $L_\infty$ covering number by  Fat-shattering Dimension}
The following proposition and lemma are standard in statistical learning theory and their proof can be found, for instance, in \citep{AlonBeCeHa93}. We provide the statement and the proof of the lemma for completeness and so that we can state it in the exact form it is used in this work.

\begin{proposition}\label{prop:VC_multiclass}
Let $\mathcal{H} \subseteq \{0,\ldots,k\}^{\mathcal{X}}$ be a class of functions with $\mathrm{fat}_2 = d$. Then, we have,
$$
\mathcal{N}_\infty(1/2,\mathcal{H},n) \le \sum_{i=0}^d {n \choose  i} k^i
$$
and specifically for $n \ge d$ this gives,
$$
\mathcal{N}_\infty(1/2,\mathcal{H},n) \le \left(\frac{e k n }{d}\right)^d \ .
$$
\end{proposition}

\begin{lemma}\label{lem:covfat}
For any function class $\mathcal{H}$ bounded by $B$ and any $\alpha >0$ such that $\mathrm{fat}_\alpha < n$, we have,
$$
\mathcal{N}_\infty(\alpha,\mathcal{H},n) \le \left(\frac{2 e B n}{\alpha\ \mathrm{fat}_\alpha(\mathcal{H})} \right)^{\mathrm{fat}_\alpha(\mathcal{H})} \ .
$$
\end{lemma}
\begin{proof}
For any $\alpha > 0$, define an $\alpha$-discretization of the $[-B,B]$ interval as  $B_\alpha = \{-B+\alpha /2 , -B + 3 \alpha /2 , \ldots, -B+(2k+1) \alpha /2, \ldots \}$ for $0\leq k$ and $(2k+1)\alpha \leq 4 B$. Also for any $a \in [-B,B]$, define $\lfloor a \rfloor_\alpha = \argmin{r \in B_\alpha} |r - a|$ with ties being broken by choosing the smaller discretization point. For a  function $h:\X\mapsto [-B,B]$ let the function $\lfloor h \rfloor_\alpha$ be defined pointwise as $\lfloor h(x) \rfloor_\alpha$, and let $\lfloor \mathcal{H} \rfloor_\alpha = \{\lfloor h \rfloor_\alpha : h\in\mathcal{H}\}$. First, we prove that $\mathcal{N}_\infty(\alpha, \mathcal{H}, \{x_i\}_{i=1}^n) \leq \mathcal{N}_\infty(\alpha/2, \lfloor \mathcal{H} \rfloor_\alpha, \{x_i\}_{i=1}^n)$. Indeed, suppose the set $V$ is a minimal $\alpha/2$-cover of $\lfloor \mathcal{H} \rfloor_\alpha$ on $\{x_i\}_{i=1}^n$. That is,
	$$
	\forall h_{\alpha} \in \lfloor \mathcal{H} \rfloor_\alpha,\ \exists \v \in V \  \mathrm{s.t.}  ~~~~ |v_i -  h_{\alpha}(x_i)| \leq \alpha/2\ .
	$$
	Pick any $h\in\mathcal{H}$ and let $h_\alpha = \lfloor h \rfloor_\alpha$. Then $\|h-h_\alpha \|_\infty \leq \alpha/2$ and for any $i \in [n]$
	$$\left|h(x_i)- v_i\right| \leq \left|h(x_i)- h_{\alpha}(x_i)\right| + \left|h_\alpha(x_i)- v_i\right| \leq \alpha,$$
	and so $V$ also provides an $L_\infty$ cover at scale $\alpha$.
	
	We conclude that $\mathcal{N}_\infty(\alpha, \mathcal{H}, \{x_i\}_{i=1}^n) \leq \mathcal{N}_\infty(\alpha/2, \lfloor \mathcal{H} \rfloor_\alpha, \{x_i\}_{i=1}^n) =  \mathcal{N}_\infty( 1/2, {\mathcal G}, \{x_i\}_{i=1}^n) $ where $\mathcal{G} = \frac{1}{\alpha} \lfloor \mathcal{H} \rfloor_\alpha$. The functions of ${\mathcal G}$ take on a discrete set of at most $\lfloor 2B/\alpha \rfloor + 1$ values. Obviously, by adding a constant to all the functions in ${\mathcal G}$, we can make the set of values to be $\{0, \ldots, \lfloor 2B/\alpha \rfloor \}$. We now apply Proposition~\ref{prop:VC_multiclass} with an upper bound $\sum_{i=0}^d {n\choose i} k^i \leq \left(\frac{ekn}{d}\right)^d$ which holds for any $n>d$. This yields  $\mathcal{N}_\infty(1/2, {\mathcal G}, \{x_i\}_{i=1}^n) \leq \left(\frac{2e B n}{\alpha \fat_{2}({\mathcal G})} \right)^{\fat_{2}({\mathcal G})}$. 
	
	It remains to prove $\fat_2({\mathcal G}) \leq \fat_\alpha (\mathcal{H})$, or, equivalently (by scaling) $\fat_{2\alpha} (\lfloor \mathcal{H} \rfloor_\alpha)  \leq \fat_\alpha (\mathcal{H})$. 
	To this end, suppose there exists a set $\{x_{i=1}^n\}$ of size $d=\fat_{2\alpha}(\lfloor \mathcal{H} \rfloor_\alpha)$ such that there is an witness  $s_1,\ldots,s_n$ with
	$$
	\forall \epsilon \in \{\pm1\}^d , \ \exists h_{\alpha} \in \lfloor \mathcal{H} \rfloor_\alpha \ \ \ \textrm{s.t. } \forall i \in [d], \  \epsilon_i (h_{\alpha}(x_i) - s_i) \ge \alpha\ .
	$$
	Using the fact that for any $h\in\mathcal{H}$ and $h_\alpha = \lfloor h \rfloor_\alpha$ we have $\|h-h_\alpha \|_\infty \leq \alpha/2$, it follows that
	$$
	\forall \epsilon \in \{\pm1\}^d , \ \exists h \in \mathcal{H} \ \ \ \textrm{s.t. } \forall i \in [d], \  \epsilon_i (h(x_i) - s_i) \ge \alpha/2\ .
	$$
	That is, $s_1,\ldots,s_n$ is a witness to $\alpha$-shattering by $\mathcal{H}$. Thus for any $\{x_i\}_{i=1}^n$, as long as $n > \mathrm{fat}_{\alpha}$
	$$\mathcal{N}_\infty(\alpha, \mathcal{H}, \{x_i\}_{i=1}^n\}) \leq \mathcal{N}_\infty( \alpha/2, \lfloor \mathcal{H} \rfloor_\alpha, \{x_i\}_{i=1}^n) \leq \left(\frac{2e B n}{\alpha \fat_{2 \alpha}(\lfloor \mathcal{H} \rfloor_\alpha)} \right)^{\fat_{2\alpha} (\lfloor \mathcal{H} \rfloor_\alpha) } \leq \left(\frac{2e B n}{\alpha \fat_\alpha}\right)^{\fat_{\alpha} (\mathcal{H}) } \ .$$
\end{proof}

\subsection{Relating Fat-shattering Dimension and Rademacher complexity}

The following lemma upper bounds the fat-shattering dimension at scale
$\epsilon \ge \hRad_n(\H)$ in terms of the Rademacher complexity of
the function class. The proof closely follows the arguments of
Mendelson \citep[discussion after Definition 4.2]{Mendelson02}.

\begin{lemma}\label{lem:fatrad}
For any hypothesis class $\H$, any sample size $n$ and any $\epsilon > \hRad_n(\H)$ we have that 
$$
\mathrm{fat}_\epsilon(\H) \le \frac{4\ n\ \hRad_n(\H)^2}{\epsilon^2} \ .
$$
\end{lemma}
In particular, if $\hRad_n(\H) = \sqrt{R/n}$ (the typical case), then
$\mathrm{fat}_\epsilon(\H) \leq 4 R/\epsilon^2$.
\begin{proof}
  Consider any $\epsilon \ge \hRad_n(\H)$. Let
  $x^*_1,\ldots,x^*_{\fat_\epsilon}$ be the set of $\fat_\epsilon$
  shattered points. This means that there exists
  $s_1,\ldots,s_{\fat_\epsilon}$ such that for any $J \subset
  [\fat_\epsilon]$ there exists $h_J \in \H$ such that $\forall i \in
  J, h_J(x_i) \ge s_i + \epsilon$ and $\forall i \not\in J, h_J(x_i)
  \le s_i - \epsilon$.  Now consider a sample $x_1,\ldots,x_{n'}$ of
  size $n' = \lceil \frac{n}{\fat_\epsilon}\rceil \fat_\epsilon$,
  obtained by taking each $x^*_i$ and repeating it $\lceil
  \frac{n}{\fat_\epsilon}\rceil$ times, i.e.~$x_i =
  x^*_{\lfloor \frac{i}{\fat_\epsilon} \rfloor}$. Now, following
  Mendelson's arguments:
\begin{align*}
\hRad_{n'}(\H) & \ge \Es{\sigma \sim \mathrm{Unif}\{\pm 1\}^{n'}}{\frac{1}{n'} \sup_{h \in \H}\left| \sum_{i=1}^{n'} \sigma_i h(x_i)\right|} \\
& \ge \frac{1}{2} \Es{\sigma \sim \mathrm{Unif}\{\pm 1\}^{n'}}{\frac{1}{n'} \sup_{h, h' \in \H}\left| \sum_{i=1}^{n'} \sigma_i (h(x_i) - h'(x_i))\right|} ~~~~~~~~~~~ \textrm{(triangle inequality)}\\
& = \frac{1}{2} \Es{\sigma \sim \mathrm{Unif}\{\pm 1\}^{n'}}{\frac{1}{n'} \sup_{h, h' \in \H}\left| \sum_{i=1}^{\fat_\epsilon} \left(\sum_{j=1}^{\lceil n/\fat_\epsilon \rceil} \sigma_{(i-1) \fat_\epsilon + j}\right) \left(h(x^*_i) - h'(x^*_i)\right)\right|}\\
& \ge \frac{1}{2} \Es{\sigma \sim \mathrm{Unif}\{\pm 1\}^{n'}}{\frac{1}{n'}\left| \sum_{i=1}^{\fat_\epsilon} \left(\sum_{j=1}^{\lceil n/\fat_\epsilon \rceil} \sigma_{(i-1) \fat_\epsilon + j}\right) \left(h_{R}(x^*_i) - h_{\overline{R}}(x^*_i)\right)\right|}
\intertext{where for each $\sigma_1,\ldots,\sigma_{n'}$, $R \subseteq
[\fat_\epsilon]$ is given by $R = \left\{i   \middle|  \mathrm{sign}\left(\sum_{j=1}^{\lceil n/\fat_\epsilon
      \rceil} \sigma_{(i - 1) \lceil n/\fat_\epsilon\rceil + j}\right)
  \ge 0 \right\}$, $h_R$ is the function in $\H$ that
$\epsilon$-shatters the set $R$ and $h_{\overline{R}}$ be the function
that shatters the complement of set $R$.}
& \ge  \frac{1}{2} \Es{\sigma \sim \mathrm{Unif}\{\pm 1\}^{n'}}{\frac{1}{n'} \sum_{i=1}^{\fat_\epsilon} \left|\sum_{j=1}^{\lceil n/\fat_\epsilon \rceil} \sigma_{(i-1) \fat_\epsilon + j}\right| 2 \epsilon }\\
& \ge  \frac{\epsilon}{n'} \sum_{i=1}^{\fat_\epsilon} \Es{\sigma \sim \mathrm{Unif}\{\pm 1\}^{n'}}{ \left|\sum_{j=1}^{\lceil n/\fat_\epsilon \rceil} \sigma_{(i-1) \fat_\epsilon + j}\right|}\\
& \ge  \frac{\epsilon\ \fat_\epsilon}{n'}  \sqrt{\frac{\lceil n/\fat_\epsilon \rceil}{2}} \hspace{2.5in} \textrm{(Khintchine's inequality)}\\
& =  \sqrt{\frac{\epsilon^2\ \fat_\epsilon}{2\ n'}}.
\end{align*}
We can now conclude that:
\begin{align*}
\fat_\epsilon \le \frac{2 n' \hRad^2_{n'}(\H)}{\epsilon^2} \le \frac{4 n \hRad^2_{n}(\H)}{\epsilon^2}
\end{align*}
where last inequality is because Rademacher complexity decreases with
increase in number of samples and $n \leq n' \le 2n$ (because $\epsilon \ge
\hRad_n(\H)$ which implies that $\fat_\epsilon < n$). \qedhere

\end{proof}

\section{Proof of Lemma \ref{lem:key}}
Recall the key lemma used in proving our main result:

{\bf Lemma \ref{lem:key}} {\em For a non-negative $H$-smooth loss $\phi$ bounded by $b$, and any function class $\H$: }
\begin{align*}
\hRad_n(\cL_\phi(r)) & \le  21 \sqrt{6 H r}\,  \log^{\frac{3}{2}}\left(64\, n\right)\, \Rad_n(\mathcal{H})
\end{align*}

As outlined in Section \ref{sec:rad}, in order to prove the Lemma \ref{lem:key},
we take the following steps: \vspace{-0.18in}
\begin{enumerate}
  \setlength{\topsep}{0pt}
  \setlength{\itemsep}{1pt}
  \setlength{\parskip}{0pt}
  \setlength{\parsep}{0pt}
\item We use Lemma \ref{lem:cover} (the refined Dudley Integral bound)
  to bound the Rademacher complexity of the empirically restricted
  loss class in terms of the $L_2$-covering numbers of the class.
\item We use smoothness to get an $r$-dependent bound on the
  $L_2$-covering numbers of the empirically restricted loss class in
  terms of $L_{\infty}$-covering numbers of the unrestricted
  hypothesis class. The key to doing this is the Lemma
  \ref{lem:smoothbnd}, which follows from the self bounding
  property, Lemma \ref{lemma:smooth}.
\item We bound the $L_{\infty}$-covering numbers of the unrestricted
  class in terms of its fat-shattering dimension (Lemma \ref{lem:covfat}), which in turn can be  bounded in terms of its Rademacher complexity (Lemma \ref{lem:fatrad}).
\end{enumerate}

We first present Lemma \ref{lem:smoothbnd}, which follows from Lemma
\ref{lemma:smooth} and is the key property we actually use.  Equipped
with this lemma and the results from Appendix \ref{app:rel} relating
the various complexity measures, we then proceed to the main proof of
Lemma \ref{lem:key}.

\begin{lemma}\label{lem:smoothbnd}
For any $H$-smooth non-negative function $f : \reals \mapsto \reals$ and any $t, r \in \reals$ we have that
$$
\left(f(t) - f(r)\right)^2 \le 6 H ( f(t) + f(r) )  (t - r)^2 \ .
$$
\end{lemma}
\begin{proof}
We start by noting that by the mean value theorem for any $t, r \in \reals$ there exists $s$ between $t$ and $r$  such that 
\begin{equation}\label{eq:meanval}
f(t) - f(r) = f'(s) (t - r) \ .
\end{equation}
By smoothness, we have that
$$
\abs{f'(s) - f'(t)} \le H\abs{t-s} \leq H \abs{t - r}.
$$
Hence we see that
\begin{equation}\label{eq:lipder}
\abs{f'(s)} \le \abs{f'(t)} + H \abs{t - r} \ .
\end{equation}
We now consider two cases: \\
\noindent{\bf Case I:} If $\abs{t - r} \le \frac{\abs{f'(t)}}{5 H}$ then
by \eqref{eq:lipder}, $\abs{f'(s)} \le 6/5 \abs{f'(t)}$,
and combining this with \eqref{eq:meanval} we have:
\begin{align}\label{eq:int1}
(f(t) - f(r))^2 & \le f'(s)^2 (t-r)^2 \leq \frac{36}{25} f'(t)^2 (t -
r)^2 \ .\notag
\intertext{But Lemma \ref{lemma:smooth} ensures $f'(t)^2 \leq 4 H
  f(t)$ yielding:}
& \le \frac{144}{25} H f(t) (t - r)^2 < 6 H f(t) (t-r)^2 \ .
\end{align}
\noindent{\bf Case II:} On the other hand, when $\abs{t - r} >
\frac{\abs{f'(t)}}{5 H}$, we have from \eqref{eq:lipder} that
$\abs{f'(s)} \le 6 H \abs{t-r}$.  Plugging this into \eqref{eq:meanval} yields:
\begin{align}
(f(t) - f(r))^2 & = \abs{f(t)-f(t)} \cdot \abs{f(t)-f(r)} \le \abs{f(t) - f(r)} (\abs{f'(s)} \abs{t-r}) \notag \\
& \le \abs{f(t) - f(r)} (6 H \abs{t-r} \cdot \abs {t - r}) = 6
H \abs{f(t)-f(r)} (t-r)^2 \notag \\
& \le 6 H \max\{f(t), f(r)\} (t - r)^2 \ .\label{eq:int2}
\end{align}
Combining the two cases, we have from \eqref{eq:int1} and
\eqref{eq:int2}) and the non-negativity of $f(\cdot)$, that in either
case:
\begin{equation*}
(f(t) - f(r))^2 \le 6 H \left(f(t) + f(r) \right) (t - r)^2 \ .\qedhere
\end{equation*}
\end{proof}

\begin{proof}[{\bf Proof of Lemma \ref{lem:key}}]
Following the outline above:
\vspace{-0.15in}

\paragraph{Bounding $\hRad_n(\cL_\phi(r))$ in terms of $\mathcal{N}_2(\cL_\phi(r))$}
Dudley's integral bound lets us bound the Rademacher complexity of a
class in terms of its empirical $L_2$ covering number. Here we use a
more refined version of Dudley's integral bound due to Mendelson
\citep{Mendelson02} and more explicitly stated in \citep{SrebroSr10} and
included here for completeness as Lemma \ref{lem:cover}:
\begin{align}\label{eq:dud}
  \hRad_n(\cL_\phi(r))\le \inf_{\alpha > 0}\left\{ 4 \alpha + 10
    \int_{\alpha}^{\sqrt{b r}}
    \sqrt{\frac{\mathcal{N}_2\left(\cL_\phi(r),\epsilon,
          n\right)}{n}} d \epsilon \right\} \ .
\end{align}

\paragraph{Bounding $\mathcal{N}_2(\cL_\phi(r))$ in terms of $\mathcal{N}_{\infty}(\mathcal{H})$}
By Lemma \ref{lem:smoothbnd} we see that for a non-negative $H$-smooth function $f$, we have that $\left(f(t) - f(r)\right)^2 \le 6 H ( f(t) + f(r) )  (t - r)^2$.  Using this inequality, for any sample
$(x_1,y_1),\ldots,(x_n,y_n)$:
\begin{align*}
& \sqrt{\frac{1}{n} \sum_{i=1}^n (\phi(h(z_i),z_i) - \phi(h_\epsilon(z_i),z_i))^2}   \le \sqrt{\frac{6 H}{n} \sum_{i=1}^n \left( \phi(h(z_i),z_i)  +  \phi(h_\epsilon(z_i),z_i)\right) (h(z_i) - h_\epsilon(z_i))^2} \\ 
& ~~~~~~~~~~~~~~~ \le \sqrt{\frac{6 H}{n} \sum_{i=1}^n \left( \phi(h(z_i),z_i)  +  \phi(h_\epsilon(z_i),z_i)\right) } \sqrt{\max_{i \in [n]} (h(z_i) - h_\epsilon(z_i))^2}\\
&~~~~~~~~~~~~~~~ \le \sqrt{12 H  r}\ \max_{i \in [n]}| h(z_i) - h_\epsilon(z_i)| \ .
\end{align*}
That is, an empirical $L_{\infty}$ cover of $\left\{h \in \H : \hL(h)
  \le r \right\}$ at radius $\epsilon/\sqrt{12 H r}$ is also an
empirical $L_2$ cover of $\cL_\phi(r)$ at radius $\epsilon$, and we
can conclude that:
\begin{equation}\label{eq:N2Ninf}
\mathcal{N}_2\left(\cL_\phi(r),\epsilon, n\right) \le \mathcal{N}_\infty\left(\left\{h \in \H : \hL(h) \le r\right\},\frac{\epsilon}{\sqrt{12 H r}}, n\right) \le \mathcal{N}_\infty\left(\H,\frac{\epsilon}{\sqrt{12 H r}}, n\right) \ .
\end{equation}

\paragraph{Bounding $\mathcal{N}_{\infty}(\mathcal{H})$ in terms of $\hRad_n(\H)$}
  Note that for any $\epsilon > \Rad_n(\mathcal{H})$, by Lemma \ref{lem:fatrad}, $\mathrm{fat}_\epsilon \le n$. Hence the $L_{\infty}$ covering number at scale $\epsilon/\sqrt{12 H r}$  can be bounded
in terms of the fat shattering dimension at
that scale using Lemma \ref{lem:covfat} as:
\begin{equation}\label{eq:NinfFat}
\mathcal{N}_\infty\left(\H,\frac{\epsilon}{\sqrt{12 H r}}, n\right) \le \left(\frac{2 e n \ \sqrt{12 H r} B }{\epsilon\ \mathrm{fat}_\frac{\epsilon}{\sqrt{12 H r}}(\mathcal{H}) }\right)^{\mathrm{fat}_\frac{\epsilon}{\sqrt{12 H r}}(\mathcal{H})}~.
\end{equation}

Hence by \eqref{eq:dud}, we have:
\begin{align}
\hRad_n(\cL_\phi(r)) & \le  4 \sqrt{12 H r} \Rad_n(\mathcal{H}) + 10 \int_{\sqrt{12 H r} \Rad_n(\mathcal{H})}^{\sqrt{b r}} \sqrt{\frac{\mathrm{fat}_\frac{\epsilon}{\sqrt{12 H r}}(\mathcal{H}) \log\left(\frac{2 e n \ \sqrt{12 H r} B }{\epsilon \mathrm{fat}_\frac{\epsilon}{\sqrt{12 H r}}(\mathcal{H}) }\right)}{n}} d \epsilon \notag 
\intertext{and, after a change of integration variable, we have:}
& \le  4 \sqrt{12 H r}\Rad_n(\mathcal{H}) + 10 \sqrt{12 H r} \int_{\Rad_n(\mathcal{H})}^{\sqrt{b/12 H}} \sqrt{\frac{\mathrm{fat}_\epsilon(\mathcal{H}) \log\left(\frac{2 e n \  B }{\epsilon \mathrm{fat}_\epsilon(\mathcal{H}) }\right)}{n}} d \epsilon \notag \\
& \le  4 \sqrt{12 H r}\Rad_n(\mathcal{H}) + 10 \sqrt{12 H r} \int_{\Rad_n(\mathcal{H})}^{\sqrt{b/12 H}} \sqrt{\frac{\mathrm{fat}_\epsilon(\mathcal{H}) \log\left(\frac{2 e  B }{\epsilon  }\right)}{n}} d \epsilon \notag \\
& ~~~~~~~~~~~~~~~~~~~~~~~~+ 10 \sqrt{12 H r} \int_{\Rad_n(\mathcal{H})}^{\sqrt{b/12 H}} \sqrt{\frac{\mathrm{fat}_\epsilon(\mathcal{H}) \log\left(\frac{n}{\mathrm{fat}_\epsilon(\mathcal{H}) }\right)}{n}} d \epsilon \ . \label{eq:3term}
\end{align}
We now bound the second term in the sum above. To this end note that for any $\epsilon \ge \Rad_n(\mathcal{H})$, bounding the fat-shattering dimension in terms of the  Rademacher complexity (Lemma \ref{lem:fatrad}) we get:
\begin{align}\label{eq:2ndterm}
10 \sqrt{12 H r} & \int_{\Rad_n(\mathcal{H})}^{\sqrt{b/12 H}} \sqrt{\frac{\mathrm{fat}_\epsilon(\mathcal{H}) \log\left(\frac{2 e  B }{\epsilon  }\right)}{n}} d \epsilon  \\
& ~~~~~~~~~~ \le  10 \sqrt{12 H r}\ \Rad_n(\mathcal{H}) \int_{\Rad_n(\mathcal{H})}^{\sqrt{b/12 H}} \frac{\sqrt{\log\left(\frac{2 e B}{\epsilon}\right)}}{\epsilon} d \epsilon \notag \\
& ~~~~~~~~~~ \le  10 \sqrt{12 H r}\ \Rad_n(\mathcal{H}) \left[ -\frac{2}{3} \log^{3/2}\left(\frac{2 e B}{\epsilon}\right)\right]_{\Rad_n(\mathcal{H})}^{\sqrt{b/12 H}}  \notag \\
& ~~~~~~~~~~\le \frac{20}{3} \sqrt{12 H r}\ \Rad_n(\mathcal{H}) \left(\log^{3/2}\left(\frac{2 e B}{\Rad_n(\mathcal{H})}\right) -  \log^{3/2}\left(\sqrt{\frac{24 e H B^2}{b}}\right) \right) \notag \\
& ~~~~~~~~~~ \le \frac{20}{3} \sqrt{12 H r}\  \log^{3/2}\left(\frac{2 e B}{\Rad_n(\mathcal{H})}\right)  \Rad_n(\mathcal{H}) \ .
\end{align}
Now we move to the third term of \eqref{eq:3term}, we further split this integral into three parts as:
\begin{align}
 \int_{\Rad_n(\mathcal{H})}^{\sqrt{b/12 H}} & \sqrt{\frac{\mathrm{fat}_\epsilon(\mathcal{H}) \log\left(\frac{n}{\mathrm{fat}_\epsilon(\mathcal{H}) }\right)}{n}} d \epsilon  \\
 &  \le  \int_{\Rad_n(\mathcal{H})}^{\gamma} \sqrt{\frac{\mathrm{fat}_\epsilon(\mathcal{H}) \log\left(\frac{n}{\mathrm{fat}_\epsilon(\mathcal{H}) }\right)}{n}} d \epsilon + \int_{\gamma}^{\theta} \sqrt{\frac{\mathrm{fat}_\epsilon(\mathcal{H}) \log\left(\frac{n}{\mathrm{fat}_\epsilon(\mathcal{H}) }\right)}{n}} d \epsilon  \notag \\
 & ~~~~~ + \int_{\theta}^{\sqrt{b/12 H}} \sqrt{\frac{\mathrm{fat}_\epsilon(\mathcal{H}) \log\left(\frac{n}{\mathrm{fat}_\epsilon(\mathcal{H}) }\right)}{n}} d \epsilon \ .\notag 
\intertext{Now let $\theta$ be such that $\mathrm{fat}_\theta > n/e$, so that for all $\epsilon > \theta$, $\log(n/\mathrm{fat}_\epsilon) \le 1$. Hence,}
& \le  \int_{\Rad_n(\mathcal{H})}^{\gamma} \sqrt{\frac{\mathrm{fat}_\epsilon(\mathcal{H}) \log\left(\frac{n}{\mathrm{fat}_\epsilon(\mathcal{H}) }\right)}{n}} d \epsilon + \int_{\gamma}^{\theta} \sqrt{\frac{\mathrm{fat}_\epsilon(\mathcal{H}) \log\left(\frac{n}{\mathrm{fat}_\epsilon(\mathcal{H}) }\right)}{n}} d \epsilon   \notag \\
 & ~~~~~ + \int_{\theta}^{\sqrt{b/12 H}} \sqrt{\frac{\mathrm{fat}_\epsilon(\mathcal{H}) }{n}} d \epsilon \ . \notag 
\intertext{Now to handle the second term in the integral note that in the range $d \in [1,n/e]$, the function $d \log\left(\frac{n}{d}\right)$ is monotonically increasing in $d$ and so in the range of $\epsilon \in [\gamma,\theta]$, $\mathrm{fat}_\epsilon \log\left(\frac{n}{\mathrm{fat}_\epsilon} \right) \le \mathrm{fat}_\gamma \log\left(\frac{n}{\mathrm{fat}_\gamma} \right)$. Thus we have that }
& \le  \int_{\Rad_n(\mathcal{H})}^{\gamma} \sqrt{\frac{\mathrm{fat}_\epsilon(\mathcal{H}) \log\left(\frac{n}{\mathrm{fat}_\epsilon(\mathcal{H}) }\right)}{n}} d \epsilon + \int_{\gamma}^{\theta} \sqrt{\frac{\mathrm{fat}_\gamma(\mathcal{H}) \log\left(\frac{n}{\mathrm{fat}_\gamma(\mathcal{H}) }\right)}{n}} d \epsilon \notag \\
 & ~~~~~ + \int_{\theta}^{\sqrt{b/12 H}} \sqrt{\frac{\mathrm{fat}_\epsilon(\mathcal{H}) }{n}} d \epsilon \ .\notag 
\intertext{Further since for all $\epsilon \in [\Rad_n(\mathcal{H}), \gamma]$ $\mathrm{fat}_\epsilon \le \mathrm{fat}_\gamma$ we have that}
& \le  \int_{\Rad_n(\mathcal{H})}^{\gamma} \sqrt{\frac{\mathrm{fat}_\epsilon(\mathcal{H}) \log\left(\frac{n}{\mathrm{fat}_\gamma(\mathcal{H}) }\right)}{n}} d \epsilon + \int_{\gamma}^{\theta} \sqrt{\frac{\mathrm{fat}_\gamma(\mathcal{H}) \log\left(\frac{n}{\mathrm{fat}_\gamma(\mathcal{H}) }\right)}{n}} d \epsilon \notag \\
 & ~~~~~ + \int_{\theta}^{\sqrt{b/12 H}} \sqrt{\frac{\mathrm{fat}_\epsilon(\mathcal{H}) }{n}} d \epsilon \ . \notag
\intertext{Since all three integrals above are in the range such that $\epsilon > \Rad_n(\mathcal{H})$,  bounding the fat-shattering dimension in terms of the  Rademacher complexity (Lemma \ref{lem:fatrad}) in the first and third integrals :}
& \le  \Rad_n(\mathcal{H})\int_{\Rad_n(\mathcal{H})}^{\gamma} \frac{\sqrt{\log\left(\frac{n}{\mathrm{fat}_\gamma(\mathcal{H}) }\right)}}{\epsilon} d \epsilon + \int_{\gamma}^{\theta} \sqrt{\frac{\mathrm{fat}_\gamma(\mathcal{H}) \log\left(\frac{n}{\mathrm{fat}_\gamma(\mathcal{H}) }\right)}{n}} d \epsilon \notag \\
 & ~~~~~ + \Rad_n(\mathcal{H}) \int_{\theta}^{\sqrt{b/12 H}} \frac{1}{\epsilon} d \epsilon \notag \\
 & \le  \Rad_n(\mathcal{H})\sqrt{\log\left(\frac{n}{\mathrm{fat}_\gamma(\mathcal{H}) }\right)} \log\left(\frac{1}{\Rad_n(\mathcal{H})}\right) + \sqrt{\frac{\mathrm{fat}_\gamma(\mathcal{H}) \log\left(\frac{n}{\mathrm{fat}_\gamma(\mathcal{H}) }\right)}{n}} \left(\gamma - \theta\right) \notag \\
 & ~~~~~ + \Rad_n(\mathcal{H}) \log\left(\frac{1}{\theta}\right)\notag\\
 \intertext{}
  & \le  \Rad_n(\mathcal{H})\sqrt{\log\left(\frac{n}{\mathrm{fat}_\gamma(\mathcal{H}) }\right)} \log\left(\frac{1}{\Rad_n(\mathcal{H})}\right) + \sqrt{\frac{\mathrm{fat}_\gamma(\mathcal{H}) \log\left(\frac{n}{\mathrm{fat}_\gamma(\mathcal{H}) }\right)}{n}} \sqrt{\frac{b}{12 H}} \notag \\
 & ~~~~~ + \Rad_n(\mathcal{H}) \log\left(\frac{1}{\Rad_n(\mathcal{H})}\right)\notag 
 \intertext{where in the last inequality we used the fact that $\gamma - \theta \le \sqrt{b/12 H}$ (integral range) and that $\theta \ge \Rad_n(\mathcal{H})$. Picking $\gamma$ to be such that $\mathrm{fat}_\gamma = 12 H n \Rad_n^2(\mathcal{H})/ b $ we conclude that}
  & \le  \Rad_n(\mathcal{H})\sqrt{\log\left(\frac{b}{12 H \Rad^2_n(\mathcal{H}) }\right)} \log\left(\frac{1}{\Rad_n(\mathcal{H})}\right) + \Rad_n(\mathcal{H}) \sqrt{\log\left(\frac{b}{12 H \Rad^2_n(\mathcal{H}) }\right)} \notag \\
 & ~~~~~ + \Rad_n(\mathcal{H}) \log\left(\frac{1}{\Rad_n(\mathcal{H})}\right) \notag\\
 & \le 3\ \Rad_n(\mathcal{H}) \log^{3/2}\left(\frac{2 e B}{\Rad_n(\mathcal{H})}\right) \ .
\end{align}
Hence plugging back the above and \eqref{eq:2ndterm} back in \eqref{eq:3term} we conclude that
\begin{align}\label{eq:removingB}
\hRad_n(\cL_\phi(r)) & \le  4 \sqrt{12 H r}\Rad_n(\mathcal{H}) + 7 \sqrt{12 H r} \Rad_n(\mathcal{H}) \log^{3/2}\left(\frac{2 e B}{\Rad_n(\mathcal{H})}\right) \\
& ~~~~~~~~~~+ 30 \sqrt{12 H r} \Rad_n(\mathcal{H}) \log^{3/2}\left(\frac{2 e B}{\Rad_n(\mathcal{H})}\right) \notag \\
& \le 41 \sqrt{12 H r} \Rad_n(\mathcal{H}) \log^{3/2}\left(\frac{2 e B}{\Rad_n(\mathcal{H})}\right) \ .
\end{align}
Now by definition of Rademacher complexity, we have,
\begin{align*}
\hRad_n(\H) & = \sup_{x_1,\ldots , x_n \in \X} \Es{\sigma \sim \text{Unif}(\{\pm 1\}^n)}{\sup_{h \in \H} \frac{1}{n} \left|\sum_{i=1}^n  h(x_i) \sigma_i \right|}\\
& \ge \sup_{x \in \X} \Es{\sigma \sim \text{Unif}(\{\pm 1\}^n)}{\sup_{h \in \H} \frac{1}{n} \left|\sum_{i=1}^n  h(x) \sigma_i \right|}\\
& = \left(\sup_{x \in \X}\sup_{h \in \H} |h(x)|\right) \  \left(\Es{\sigma \sim \text{Unif}(\{\pm 1\}^n)}{ \frac{1}{n} \left|\sum_{i=1}^n \sigma_i\right|}\right)\\
& = B\ \Es{\sigma \sim \text{Unif}(\{\pm 1\}^n)}{ \frac{1}{n} \left|\sum_{i=1}^n \sigma_i\right|} \ge \frac{B}{\sqrt{2n}}
\end{align*}
where the last step is due to Khintchine's inequality (see, e.g., page 364 of \citep{CesaBianchiLu06}). Thus we see that $\frac{2 e B}{\Rad_n(\mathcal{H} }\le 8 \sqrt{n}$. Plugging this in \eqref{eq:removingB}, we conclude that
\begin{align*}
\hRad_n(\cL_\phi(r)) &  \le 21 \sqrt{6 H r}\ \log^{\frac{3}{2}}\left(64\ n\right) \ \Rad_n(\mathcal{H}) \ .\qedhere
\end{align*}
\end{proof}